\documentclass[twoside]{article}

\usepackage[preprint]{aistats2027}

\usepackage[round]{natbib}

\usepackage[utf8]{inputenc} 
\usepackage[T1]{fontenc}    
\usepackage{hyperref}       
\usepackage{url}            
\usepackage{booktabs}       
\usepackage{amsfonts}       
\usepackage{nicefrac}       
\usepackage{microtype}      
\usepackage{xcolor}         

\usepackage{titletoc}

\usepackage{amsmath}
\usepackage{amssymb}
\usepackage{mathtools}
\usepackage{amsthm}
\usepackage{bm}
\usepackage{enumitem}

\usepackage{tikz}
\usetikzlibrary{positioning, arrows.meta, fit, backgrounds, calc, shadows, decorations.pathreplacing}

\usepackage{blkarray}
\usepackage{caption}
\usepackage{multirow}

\usepackage{float} 

\theoremstyle{plain}
\newtheorem{theorem}{Theorem}[section]

\newtheorem{lemma}[theorem]{Lemma}

\theoremstyle{definition}
\newtheorem{definition}[theorem]{Definition}
\newtheorem{assumption}[theorem]{Assumption}
\theoremstyle{remark}
\newtheorem{remark}[theorem]{Remark}

\definecolor{mydarkblue}{rgb}{0, 0.08, 0.45}
\hypersetup{
    colorlinks,
    linkcolor={mydarkblue},
    citecolor={cyan!80!black},
    urlcolor={magenta!60!black}
}

\usepackage[most]{tcolorbox}

\definecolor{acadblue}{RGB}{238,244,255}

\tcbset{
    blue_style/.style={
        enhanced,
        colback=acadblue,
        colframe=black,
        boxrule=0.5pt,
        arc=2pt,
        auto outer arc,
        boxsep=0pt,
        left=10pt,
        right=10pt,
        top=2pt,
        bottom=2pt,
        before skip=10pt,
        after skip=10pt
    }
}

\begin{document}

\runningtitle{Minimax Additive Regression under Unknown Dependent Designs}

\runningauthor{Ferrere, Gamboa, Loubes}

\twocolumn[

\aistatstitle{Minimax Additive Regression under Unknown Dependent Designs}

\aistatsauthor{
Baptiste Ferrere\,$^\ast$
\And
Fabrice Gamboa
\And
Jean-Michel Loubes
}

\aistatsaddress{
Université de Toulouse \\ EDF R\&D
\And
Université de Toulouse \\ ANITI
\And
Université de Toulouse \\ ANITI
} ]

\begingroup
\renewcommand{\thefootnote}{\fnsymbol{footnote}}
\footnotetext[1]{%
Corresponding author:
\href{mailto:baptiste.ferrere@edf.fr}
{\texttt{baptiste.ferrere@edf.fr}}%
}
\endgroup

\begin{abstract}
We study additive regression under an unknown and potentially non product design distribution, allowing the number of covariates to grow with the sample size. We consider a coupled class that separately controls the smoothness of the marginal densities and of each additive component multiplied by the corresponding marginal density. Under joint-density bounds that hold uniformly in the dimension and suitable dimension-growth conditions, we establish matching minimax bounds for prediction. With known marginal densities, the classical additive rate is attainable. When the marginals are unknown, this rate is preserved if the densities are at least as smooth as the weighted components. Otherwise, marginal-density smoothness determines the minimax rate over the coupled class. Finally, we recover all additive components with total squared error of the same order as the prediction error.
\end{abstract}

\section{Introduction}

Additive regression combines nonparametric flexibility with an interpretable
representation in terms of univariate components
\citep{stone1985additive,hastie1986generalized}. Classical formulations impose
component smoothness in fixed function spaces, while the design distribution
determines the centering constraints and prediction geometry. A different
statistical problem arises when smoothness is defined through coordinates
that themselves depend on the design. The distribution then also determines
the class of admissible regression functions. We study how uncertainty about
these coordinates affects estimation when the covariates may be dependent.

We consider the random-design regression model
\begin{equation}
    Y_i \coloneqq f^{\star}(\mathbf X_i)+\xi_i,
    \qquad i=1,\ldots,n,
    \label{eq:regression-model}
\end{equation}
where $\mathbf X_1,\ldots,\mathbf X_n$ are independent copies of a random
vector $\mathbf X\coloneqq(X_1,\ldots,X_d)$ supported on $[0,1]^d$, with
distribution $P$ and continuous density $p$. The coordinates of $\mathbf X$
may be dependent. The noise variables are independent
$\mathcal N(0,\sigma^2)$ random variables, independent of the covariates.
We allow $d$ to grow with $n$ and seek to estimate $f^{\star}$ under
squared $L^2(P)$ loss. We assume an additive structure,
\begin{equation}
    f^{\star}(\mathbf x)
    = f_0^{\star}+\sum_{j=1}^d f_j^{\star}(x_j),
    \qquad \mathbf x\in[0,1]^d,
    \label{eq:additive-model}
\end{equation}
with the population-centering constraints
\begin{equation}
    \mathbb E_P[f_j^{\star}(X_j)]=0,
    \qquad j=1,\ldots,d.
\end{equation}
These constraints align identifiability with the prediction distribution
and imply $f_0^{\star}=\mathbb E_P[f^{\star}(\mathbf X)]$.
Throughout, we study the centered case $f_0^{\star}=0$.

Our starting point is the inverse-marginal-density representation of
\citet{ferrere2026generalized}, whose additive blocks provide explicit
coordinates satisfying these centering constraints. A Sobolev constraint
on the representation coefficients is equivalent to a periodic Sobolev
constraint \citep{tsybakov2009introduction} on each
\emph{density-weighted} component $p_jf_j^{\star}$, where $p_j$ is the
marginal density of $X_j$. We study a coupled class in which these weighted
components have Sobolev regularity $\beta$ and the marginal densities have
H\"older regularity $\gamma$. This formulation distinguishes regularity
within each fixed representation from uncertainty about the densities
that define it. Our goal is to characterize the resulting minimax
prediction risk when the marginals are unknown.

\paragraph{Contributions.}
Our main contribution is to characterize when unknown marginal densities preserve the classical additive rate and when their regularity becomes the limiting factor, through matching upper and lower bounds (Theorems~\ref{thm:main-unknown-density-upper}, \ref{thm:unknown-density-lower-smooth}, and~\ref{thm:unknown-density-lower-rough}). Theorem~\ref{thm:representation} specializes the existing representation to our additive setting. Theorem~\ref{thm:riesz} then gives explicit Riesz bounds with dimension-independent constants when the joint density is bounded above and away from zero uniformly in $d$. This assumption is restrictive, but ensures that dependence does not cause the stability constants to deteriorate as the dimension grows.

With known marginals, the expected error is the classical additive rate $dn^{- \frac{2\beta}{2\beta + 1} }$ \citep{stone1985additive}. Once stability is established, balancing series approximation and least-squares variance yields the upper bound in Theorem~\ref{thm:ez_upper}, while a standard Assouad argument gives the matching lower bound in Theorem~\ref{thm:ez_lower} for every admissible fixed design distribution. The matching rates hold when $d\log n=o(n^{ \frac{2\beta}{2\beta + 1} })$ (Assumption \ref{assu:dimension}).

For unknown marginals, our upper bound requires controlling an estimated regression dictionary while preserving linear dependence on $d$. Theorem~\ref{thm:main-unknown-density-upper} addresses this through sample splitting: one subsample estimates the marginal densities, and the other fits the regression function by least squares. The central step is to clip the density estimates and center the dictionary under Lebesgue measure, with an intercept absorbing the change in centering. This yields a uniformly conditioned population Gram matrix for every realization of the clipped estimates, without requiring uniform closeness to the true densities. The same centering makes the coordinatewise residuals orthogonal under the uniform product measure, preserving a factor $d$ in the prediction bound. Lemma~\ref{lem:approximation} controls the approximation error using pointwise mean-squared density estimation bounds, uniformly over evaluation points. Combined with standard least-squares and matrix-concentration bounds, this yields the prediction upper rate
\begin{equation}
    \underbrace{d\,n^{-\frac{2\beta}{2\beta+1}}}
                _{\text{series regression}}
    +
    \underbrace{d\,n^{-\frac{2\gamma}{2\gamma+1}}}
                _{\text{marginal-density estimation}},
\end{equation}
under the same dimension-growth condition.

To establish optimality, Theorem~\ref{thm:unknown-density-lower-smooth} first recovers the classical additive lower rate by restricting the model to the uniform design. This bound holds for every $\gamma>0$ and matches the upper rate when $\gamma\geq\beta$. Thus, sufficiently smooth unknown marginals preserve the known-marginal rate.

When $\gamma<\beta$, a density-estimation lower bound alone does not suffice: the responses also help distinguish the alternatives, which must remain in the coupled regression class. Theorem~\ref{thm:unknown-density-lower-rough} resolves this difficulty by varying the marginals while keeping the weighted components fixed. A bounded sinusoidal coupling of localized perturbations preserves the joint density bounds independently of $d$. Dividing a fixed smooth, mean-zero function by each perturbed marginal density preserves both weighted Sobolev regularity and population centering. An Assouad argument then combines control of the KL divergence between the full observation laws with separation between the regression functions. This yields the lower rate of $dn^{- \frac{2\gamma}{2\gamma + 1} }$ when $d=o(n^{\frac{2\gamma}{2\gamma + 1} })$ (Assumption~\ref{assu:dimension_2}), a condition that also ensures the upper bound applies in this regime. The matching bounds establish that the slowdown is intrinsic to the coupled class, rather than a general penalty for unknown designs in classes with fixed component smoothness.

Finally, accurate prediction of an additive sum does not by itself ensure
accurate recovery of its components. Theorem~\ref{thm:component_recovery}
uses stability of the fitted coefficients to recover all population-centered
components without an additional regression fit. Their aggregate
squared-error bound has the same order as the global prediction bound.

\paragraph{Related work.}
The closest statistical precedents are backfitting and smooth-backfitting
methods for random-design additive regression
\citep{mammen1999existence,horowitz2006optimal}. Under suitable assumptions,
these methods estimate individual components at univariate oracle rates
without requiring independent covariates. Their oracle comparison concerns
knowledge of the other components, whereas our reference problem assumes
knowledge of the marginal densities. More fundamentally, they impose
smoothness directly on the components; our analysis concerns the regularity
of their density-weighted counterparts.

Population centering is the first-order instance of hierarchical
orthogonality in functional ANOVA
\citep{stone1994use,huang1998projection,hooker_2007,chastaing_generalized_2012}.
Under dependence, distinct component spaces need not be mutually orthogonal.
\citet{ferrere2026generalized} characterize these spaces through a
distribution-adapted Riesz representation and provide a proof-of-concept
estimator. Our work supplies minimax guarantees for its additive
specialization, including the effect of estimating the marginal densities.

In high-dimensional settings, sparse additive models combine componentwise
regularity with variable selection \citep{meier2009high,ravikumar2009sparse}.
Minimax results are available for Sobolev and kernel classes
\citep{raskutti2012minimax,yuan2016minimax,tan2019doubly,
haris2022generalized}, as well as for low-order interaction models
\citep{bhattacharya2024deep}. We study a growing number of components
without imposing sparsity.

Additive models also support interpretable prediction, including neural
and tree-based approaches
\citep{rudin2019stop,agarwal2021neural,radenovic2022neural,
chang2022nodegam,nori2019interpretml,lengerich2020purifying,benard2025tree}.
This motivates controlling component estimation error alongside the
prediction error of their sum.

\section{Preliminaries}\label{sec:preliminaries}

\paragraph{Notation.}
Let $\lambda$ denote the Lebesgue measure on $[0,1]$. For any $j \in \{1,\dots,d\}$, we denote by $P_j$ the marginal distribution of $X_j$ and by $p_j$ the corresponding density, such that $ \frac{d P_j}{d \lambda} = p_j $. We denote by $L^2(P)$ the Hilbert space of square-integrable and $P-$measurable functions. For simplicity, we denote by $\| \cdot \|$ the $L^2(P)-$norm and $\| \cdot \|_2$ the Euclidean norm. We denote by $\lfloor \cdot \rfloor$ the integer part function. Throughout the paper, for two nonnegative quantities $a_{n,d}$ and $b_{n,d}$, we write $a_{n,d}\lesssim b_{n,d}$ if there exists a constant $C>0$, independent of $n$ and $d$, such that $a_{n,d}\leq Cb_{n,d}$. We define $\gtrsim$ analogously and write $a_{n,d}\asymp b_{n,d}$ when both inequalities hold.

\subsection{Identifiability condition}

Without further condition, the additive representation in \eqref{eq:additive-model} is not unique, since constants can be transferred between the intercept and the component functions. Hence, to deal with an identifiable model, two structural assumptions are commonly made. A first possibility is the \emph{Lebesgue centering} $\int_0^1 f_j(x)\,dx=0,$ for all $j \in \{1,\dots,d\}$. It is natural in Gaussian white-noise formulations and when smoothness is defined in the fixed $L^2(\lambda)$ geometry \citep{tsybakov2009introduction,bhattacharya2024deep,liu2025minimaxoptimal}. In this setting, it is natural to assume some Sobolev-type regularity where each component $f_j$ can be expanded in an orthonormal basis of $L^2(\lambda)$. Random-design additive regression instead commonly assumes the following identifiability constraint:
\begin{assumption}[Population centering]\label{assu:population-centering}
    \begin{equation}
\mathbb E_{P}[f_j(X_j)] = 0, \quad j=1,\ldots,d,
\end{equation}
\end{assumption}
see for example \citet{mammen1999existence,ravikumar2009sparse,
raskutti2012minimax,horowitz2006optimal}. This condition is more natural since it is formulated in the Hilbert space equipped with the distribution of the true data. Of course it recovers the Lebesgue centering when each feature is uniform on $[0,1]$. Under Assumption \ref{assu:population-centering}, the functions $f_j$ act as the main effects of a generalized functional ANOVA decomposition \citep{stone1994use,huang1998projection,hooker_2007}.
More fundamentally, let $\mathcal H_0$ be the Hilbert space of constant functions of $L^2(P)$ and, for every $j\in \{1,\dots, d\}$, define the first-order generalized ANOVA space
\begin{equation}
\mathcal H_j \coloneqq
\left\{ h\in L^2(P_j) \; \middle| \; \mathbb E_P[h(X_j)]=0 \right\}.
\label{eq:first-order-spaces}
\end{equation}
Each $\mathcal H_j$ is a closed Hilbert subspace of $L^2(P)$. Under some conditions \citep{stone1994use,huang1998projection,hooker_2007,chastaing_generalized_2012} ensuring existence and uniqueness of the generalized functional ANOVA decomposition---satisfied in particular under Assumption \ref{assu:boundness} defined later---the first-order component spaces form a direct sum:
\begin{equation}
\mathcal H \coloneqq \mathcal H_1 \oplus \dots \oplus \mathcal H_d
\label{eq:additive-direct-sum}
\end{equation}
Here, the symbol $\oplus$ emphasizes that the decomposition of any element of $\mathcal H$ by a sum over the $\mathcal H_j$'s is unique but not necessarily orthogonal. Consequently, studying centered additive functions satisfying Assumption~\ref{assu:population-centering} is equivalent to studying elements of $\mathcal H$.

\begin{remark}
If $P$ is a product distribution, the component spaces
$\mathcal H_1,\ldots,\mathcal H_d$ are mutually orthogonal, and
\eqref{eq:additive-direct-sum} becomes an orthogonal direct sum. Under a
dependent design, each $\mathcal H_j$ remains orthogonal to
$\mathcal H_0$ by population centering, but distinct first-order
spaces are generally not mutually orthogonal.
\end{remark}
Except in the uniform setting, where standard orthonormal bases provide an immediate representation, the joint geometry of these component spaces lacked an explicit stable coordinate system. To the best of our knowledge, \citet{ferrere2026generalized} are the first to provide such coordinate systems through a distribution-adapted Riesz basis based on what they call \emph{inverse marginal-density weighting}. We build on this Riesz basis to define the regularity classes used throughout this paper.

\subsection{Representation of additive components}

In order to adapt the construction of \citet{ferrere2026generalized}, we introduce the following boundedness condition on the joint density $p$.
\begin{assumption}\label{assu:boundness}
    Let $p$ be a probability density on $[0,1]^d$. We denote by $p_{\min} \coloneqq \inf p$ and $ p_{\max} \coloneqq \sup p $. We say that $p$ satisfies this assumption if 
    \begin{equation}
        0 <p_{\min} \leq p \leq p_{\max} < + \infty ,
    \end{equation}
    and these bounds are assumed to hold uniformly in $d$.
\end{assumption}
Let further denote by $(\phi_m)_{m \in  \mathbb N}$ the trigonometric basis on $[0,1]$ with $\phi_0 = 1$ (see for example \citet{tsybakov2009introduction}). We can now define the object that will be central to our class of regularity.
\begin{definition}
    Let $j \in \{1,\dots,d\}$ and $m$ be a positive integer, we define the function $\psi_m^{(j)}$ as
    \begin{equation}
        \psi_m^{(j)}(x) \coloneqq \frac{ \phi_m(x) }{ p_j(x) }, \quad x \in [0,1].
    \end{equation}
    We introduce the notation $ \Psi \coloneqq \left( (\psi_m^{(j)})_{m \geq 1} \right)_{ j=1,\dots,d } $.
\end{definition}
We have the first technical result induced by $\Psi$.

\begin{theorem}\label{thm:riesz}
    Let $p$ be a density that satifies Assumption~\ref{assu:boundness}.
    For any finite sequence of real numbers $\bm a \coloneqq \left( (a_m^{(j)})_{m \geq 1} \right)_{j=1,\dots,d}$, we define the function $f_{\bm a}$:
    \begin{equation}
        f_{\bm a} \coloneqq \sum\limits_{j=1}^d \sum\limits_{m \geq 1} a_m^{(j)}  \psi_m^{(j)}.
    \end{equation}
    The following inequality holds:
    \begin{equation}\label{eq:riesz_ineq}
    C_{\min}^p \| \bm a \|_2^2 \leq \| f_{\bm a} \|^2 \leq C_{\max}^p \| \bm a \|_2^2,
    \end{equation}
    with $C_{\min}^p \coloneqq p_{\min}/p_{\max}^2$ and $ C_{\max}^p \coloneqq p_{\max}/p_{\min}^2 $. We say that $\Psi$ is a Riesz sequence \citep{brezis2011functional}.
\end{theorem}
This theorem will be useful to establish the upper bound inequalities. Indeed, it is a compatibility condition between the $L^2(P)$ norm and the Euclidean (and by extension the $\ell_2$) norm. Furthermore, it is a key property to obtain the next result, which charaterizes the Hilbert spaces $\mathcal H_j$.

\begin{theorem}[Representation theorem]\label{thm:representation}
    For any $f \in \mathcal H$, there exists a unique sequence of real coefficients denoted $\bm \theta \coloneqq ( ( \theta_m^{(j)} )_{m \geq 1} )_{j=1,\dots,d}$ such that each additive component is uniquely represented as
\begin{equation}\label{eq:representation}
    \forall j \in \{1,\dots,d\}, \quad f_j = \sum\limits_{m=1}^{+ \infty}\theta_m^{(j)} \psi_m^{(j)}.
\end{equation}
\end{theorem}
Note that the result of the Theorem \ref{thm:riesz} is not stated or proved in the paper of \citet{ferrere2026generalized}, however the result of Theorem \ref{thm:representation} is just an adaptation by replacing Legendre polynomials on $[-1,1]$ by the trigonometric basis on $[0,1]$.

\subsection{Smoothness classes}

We distinguish signal regularity at a fixed and known design distribution from the regularity of the marginal densities used to construct the representation.

\begin{definition}[Sobolev ball]
\label{def:coefficient-ellipsoid}
For $\beta>0$ and $R>0$, define
\begin{equation}
\Theta(\beta,R)
:=
\left\{
(\theta_m)_{m\geq1} \; \middle| \;
\sum_{m\geq 1}
m^{2\beta}\theta_m^2
\leq R^2
\right\}.
\label{eq:coefficient-ellipsoid}
\end{equation}
\end{definition}

Let $f\in\mathcal H$. By Theorem~\ref{thm:representation},
each component $f_j$ has unique Riesz coefficients $\bm \theta$ satisfying the decomposition
\begin{equation}\label{eq:weighted-component}
    g_j \coloneqq p_jf_j = \sum_{m=1}^{\infty}\theta_m^{(j)}\phi_m.
\end{equation}
Observe that $\theta_m^{(j)} = \langle g_j,\phi_m\rangle_{L^2(\lambda)}$ is the Fourier coefficient of frequency $m$ of $g_j$ in the trigonometric basis. Consequently, requiring $(\theta_m^{(j)})_{m\geq1}\in\Theta(\beta,R)$ is exactly a Sobolev ellipsoid constraint on the function $g_j$. The corresponding weighted coefficient norm is equivalent to the usual $H^\beta_{\mathrm{per}}([0,1])$ norm on its mean-zero subspace \citep{tsybakov2009introduction}.

\begin{tcolorbox}[blue_style]
\paragraph{Design-dependent smoothness.}
Our smoothness class is defined in coordinates adapted to population centering. Writing $g_j=p_jf_j$ turns $\mathbb E_P[f_j(X_j)]=0$ into $\int_0^1 g_j(x)\,dx=0$. Since the Riesz coefficients of $f_j$ are the Fourier coefficients of $g_j$, our ellipsoid constraint imposes periodic Sobolev regularity on the density-weighted component. When the marginal densities are uniform, this recovers the usual periodic additive Sobolev class. For a general design, the component class is obtained through $f_j=g_j/p_j$, so the ordinary smoothness of $f_j$ may differ from that of $g_j$. Imposing regularity on $g_j$ is therefore a modeling choice, rather than a consequence of population centering. This choice separates regularity within a fixed representation from the regularity of the densities defining it. The resulting coupled class allows us to determine when the estimation rate available with known marginals can be retained when those marginals are unknown.
\end{tcolorbox}

\begin{definition}[Design-indexed signal class]
\label{def:signal-class}
For a fixed admissible density $p$, we denote by $\mathcal F_{\beta , R}(p)$ the strict subset of $\mathcal H$ that contains every function of $f$ of the form
\begin{equation} f = \sum\limits_{j=1}^d \sum\limits_{m =1}^{+ \infty} \theta_m^{(j)} \psi_m^{(j)}, 
\end{equation} 
where for all $j=1,\dots,d$, the sequence $(\theta_m^{(j)})_m$ belongs to the ball $\Theta(\beta , R)$.
\end{definition}

This class consists of population-centered additive functions whose
weighted components $p_jf_j$ live in a \emph{standard} Sobolev ellipsoid.
The radius $R$ and the smoothness $\beta$ are the same for each additive component. The dependence on the design density is part of the definition of the function class. To quantify the error incurred when estimating the marginal densities, we impose a separate Hölder regularity condition.

\begin{definition}[Hölder regularity]\label{def:Holder}
Let $\gamma>0$, $L>0$, and $s \coloneqq \lfloor\gamma\rfloor$. A function $q:[0,1]\to\mathbb R$ is said to be $(\gamma,L)$-Hölder if $q\in C^s([0,1])$ and for all $x,y \in [0,1]$
\begin{equation}
    |q^{(s)}(x)-q^{(s)}(y)| \le L|x-y|^{\gamma-s}.
\end{equation}
\end{definition}
To define the class of all distributions of interest $p$, we need to parametrize it entirely with global and uniform constants.
\begin{definition}[Distribution class]
    Let $\gamma > 0$ be a smoothness parameter and $L>0$ a positive constant. Let also $\kappa_{\min}$ and $\kappa_{\max}$ be positive constants and free of the dimension $d$ such that $\kappa_{\min} < 1 < \kappa_{\max}$. We define the regularity class $\mathcal P_{\gamma,L}$ (where we omit the dependence in $\kappa_{\min}, \kappa_{\max}$ for simplicity) as the set of all densities $p$ defined on $[0,1]^d$ which satisfy:
    \begin{enumerate}
        \item $p_j$ is $(\gamma,L)-$Hölder for all $j \in \{1,\dots,d\}$,
        \item $ \kappa_{\min} \leq p_{\min} \leq p \leq p_{\max} \leq \kappa_{\max} $
    \end{enumerate}
\end{definition}
Note that the second condition is a uniform generalization of Assumption \ref{assu:boundness} where all the considered densities have the same bounds. Moreover, we can define the Riesz constants by $C_{\min} \coloneqq \frac{ \kappa_{\min} }{ \kappa_{\max}^2 }$ and $C_{\max} \coloneqq \frac{ \kappa_{\max} }{ \kappa_{\min}^2 } $ which are independent of $d$ and $n$ and uniform over all the considered distributions $p$.
\begin{remark}
    Note that the class $\mathcal P_{\gamma , L}$ for any $(\gamma , L)$ and any $\kappa_{\min} < 1 < \kappa_{\max}$ is not empty since the uniform density $p_{\text{unif}} \equiv 1$ always belongs to this class.
\end{remark}

We finally introduce the following Assumption, which allows the dimension $d$ to grow but at a rate slightly slower than the nonparametric rate of $n^{\frac{2 \beta}{ 2 \beta + 1 }}$.
\begin{assumption}[Growing dimension.]\label{assu:dimension}
    We suppose that $d$ can grow with $n$ at the following rate
    \begin{equation}
        d \, \log n = o\!\left( n^{ \frac{2\beta}{2\beta + 1}} \right).
    \end{equation}
\end{assumption}

\begin{definition}[Coupled parameter class]
\label{def:coupled-class}
For $\beta,\gamma,R,L,\kappa_{\min},\kappa_{\max}>0$, define the coupled class
\begin{equation}
\mathcal C_{\beta,\gamma}
:=
\bigcup_{p\in\mathcal P_{\gamma,L}}
\{p\}\times\mathcal F_{\beta,R}(p).
\label{eq:coupled-class}
\end{equation}
For the sake of simplicity, we explicit only the dependance in $\gamma$ and $\beta$.
\end{definition}
This is the global class considered in our statistical analysis. The parameter $\beta$ controls the spectral regularity of the signal in the Riesz representation associated with each fixed $p$, whereas $\gamma$ controls the regularity of the marginal densities defining that representation. Each admissible design therefore indexes its own signal class, and estimation guarantees are established uniformly over the
resulting pairs $(p,f)$.

\section{Known density regime}\label{sec:fixed_density}

Let $p$ be a density satisfying Assumption~\ref{assu:boundness}, and suppose that its marginal densities are known. Write $\mathcal F=\mathcal F_{\beta,R}(p)$ and define
\begin{equation}
    \mathfrak M(n,d,\mathcal F)
    =
    \inf_{\widehat f}\sup_{f\in\mathcal F}
    \mathbb E_{p,f}\!\left[\|f-\widehat f\|^2\right],
\end{equation}
where estimators may depend on $p$, the norm is taken in $L^2(P)$,
and the expectation averages over both the random design and the
regression noise. The constants depend on the Riesz bounds in Theorem \ref{thm:riesz}, which explicitely depend of the distribution $p$ but not of $n$ or $d$. 

\subsection{Minimax upper bound}

In order to obtain an upper bound of the minimax risk, we construct an estimator that achives the desired rate of $d \, n^{ - \frac{2\beta}{ 2\beta + 1 } }$. For an integer $M\geq1$, let $\boldsymbol\psi_M(\mathbf x)\in\mathbb R^{dM}$ collect the functions $\psi_m^{(j)}(x_j)$, ordered by the $d$ coordinates and $M$ frequencies. Let $\boldsymbol\Psi_M$ be the design matrix with rows $\boldsymbol\psi_M(\mathbf X_i)^\top$, and define
\[
    \boldsymbol \Gamma_M
    =
    \mathbb E_p\!\left[
        \boldsymbol\psi_M(\mathbf X)
        \boldsymbol\psi_M(\mathbf X)^\top
    \right],
    \qquad
    \widehat{\boldsymbol \Gamma}_M
    =
    \frac{1}{n}\boldsymbol\Psi_M^\top\boldsymbol\Psi_M.
\]
The Riesz property implies
\[
    C_{\min}^p\boldsymbol I_{dM}
    \preceq \boldsymbol \Gamma_M
    \preceq C_{\max}^p\boldsymbol I_{dM}.
\]
To ensure that least squares is well defined, introduce
\[
    \zeta_n
    =
    \left\{
        \lambda_{\min}(\widehat{\boldsymbol \Gamma}_M)
        \geq C_{\min}^p/2
    \right\}.
\]
Finally, we write $\boldsymbol Y=(Y_1,\ldots,Y_n)^\top$.

    With the previous notations, we define the least squares estimator $\widehat{f}_M$ for all $\mathbf x \in [0,1]^d$ as follows
    \begin{equation}\label{eq:oracle-thresholded-estimator}
    \widehat f_M(\mathbf x) \coloneqq \boldsymbol\psi_M(\mathbf x)^\top
\widehat{\boldsymbol\theta}_M,
\end{equation}
where $\widehat{\boldsymbol\theta}_M \coloneqq (\boldsymbol\Psi_M^\top\boldsymbol\Psi_M)^{-1} \boldsymbol\Psi_M^\top\boldsymbol Y \, \mathbf{1}_{\zeta_n} \in \mathbb{R}^{dM}$.

Note that in this definition, the inverse is evaluated only on the $\zeta_n$.
The estimator uses the known marginal densities and the lower
Riesz bound, without regularization. We show in Appendix \ref{appendix:upper_facile}, that this construction leads to the desired upper bound formalized in the next theorem.
\begin{theorem}\label{thm:ez_upper}
under Assumption \ref{assu:dimension}, we have the following minimax optimal upper bound
    \begin{equation}
        \mathfrak M(n,d,\mathcal F) \lesssim d \, n^{ - \frac{2 \beta}{ 2 \beta + 1 } },
    \end{equation}
    where all the underlying constants behind the symbol $\lesssim$ depend only on $\beta, R , \sigma , p_{\min} , p_{\max}$.
\end{theorem}

\subsection{Minimax lower bound}

For $\bm \omega\in\{-1,1\}^{dM}$, we consider the standard following construction
\[
    f_{\bm\omega}
    =
    a_M\sum_{j=1}^d\sum_{m=M+1}^{2M}
    \omega_{jm}\psi_m^{(j)},
    \qquad
    a_M=a_0M^{-\beta-1/2}.
\]
For sufficiently small $a_0>0$, all these functions belong to
$\mathcal F$ and the lower Riesz inequality bounds their squared $L^2(P)$ separation from below. By  utilizing Assouad's lemma \citep{tsybakov2009introduction}, we obtain the following theorem, which is proved in Appendix \ref{appendix:lower_facile}.

\begin{theorem}\label{thm:ez_lower}
    There exists an underlying constant which depends only on $\beta, R , \sigma , p_{\min}, p_{\max}$ such that
    \begin{equation}
        \mathfrak M(n,d,\mathcal F) \gtrsim  d \, n^{ - \frac{2 \beta}{ 2 \beta + 1 } }.
    \end{equation}
\end{theorem}

\section{Unknown density regime}\label{sec:unknown_density}

We now consider estimation over the coupled class
$\mathcal C_{\beta,\gamma}$ when the design distribution is unknown.
Throughout this class, the joint densities satisfy
$\kappa_{\min} \leq p \leq \kappa_{\max}$, with the same known
constants. We define the minimax risk by
\begin{equation}\label{eq:unknown-density-minimax-risk}
    \mathfrak M(n,d,\mathcal C_{\beta,\gamma})
    \coloneqq
    \inf_{\widehat f}
    \sup_{(p,f)\in\mathcal C_{\beta,\gamma}}
    \mathbb E_{p,f}\!\left[
        \|f-\widehat f\|^2
    \right],
\end{equation}
where the infimum ranges over all estimators measurable with respect to
the observed sample. These estimators may depend on the fixed parameters
defining the class, but not on the unknown pair $(p,f)$.
For each such pair, $\mathbb E_{p,f}$ averages over the independent
covariates $\mathbf X_1,\ldots,\mathbf X_n$ with common distribution $P$ of density $p$ and the independent Gaussian regression noise. The prediction loss is evaluated under that same distribution $P$.

\subsection{Minimax upper bound}\label{subsec:unkown_upper}

Unlike the fixed-design-distribution setting, the supremum in
\eqref{eq:unknown-density-minimax-risk} allows both the design density
$p$ and the regression function $f\in\mathcal F_{\beta,R}(p)$ to vary.
Thus, the sampling distribution, the prediction norm, and the admissible
signal class all depend on the pair under consideration. The smoothness
parameter $\beta$ controls the weighted components $p_jf_j$, whereas
$\gamma$ controls the marginal densities $p_j$. Our estimator uses sample splitting: the first subsample estimates the marginal densities, and the second fits the regression function in the resulting estimated dictionary by thresholded least squares. In the following, we take a pair $(p,f) \in \mathcal C_{\beta , \gamma}$ and provide a generic estimator that is data driven and independent of the particular choice of this pair. We fix the representation of $f$ using the equation \eqref{eq:representation}.

\paragraph{Construction of the estimator.}
Split the sample into two independent subsamples of sizes $n_0=\lfloor n/2\rfloor$ and $n_1=n-n_0$. Write $\mathcal D_n^0$ for the first-subsample covariates, $\mathcal D_n^1$ for the second one and $\mathcal D_n$ for all observed covariates. Using $\mathcal D_n^0$, one can construct marginal-density estimators $\widehat p_j$ satisfying $ \kappa_{\min} \leq \widehat p_j \leq \kappa_{\max} $ and
\begin{equation}\label{eq:main-density-guarantee}
    \max_{1\leq j\leq d}\sup_{t\in[0,1]}
    \mathbb E\!\left[
        |\widehat p_j(t)-p_j(t)|^2
    \right]
    \leq \rho_{n_0},
\end{equation}
uniformly over $(p,f)\in\mathcal C_{\beta,\gamma}$, where
$\rho_{n_0}\lesssim n_0^{- \frac{2\gamma}{2\gamma + 1} }$. Such estimators can be obtained by piecewise polynomial projection followed by clipping; the construction and its risk bound are given in Appendix.

For all $j \in \{1,\dots,d\}$ and $m \geq 1$, we define the functions $\widehat \psi_m^{(j)}$ which replace the unknown densities $p_j$ by their respective estimator and debiase the centering error due to the estimation of the densities.
\begin{equation}\label{eq:main-estimated-dictionary}
    \widehat \psi_m^{(j)}(x) \coloneqq \frac{\phi_m(x)}{\widehat p_j(x)} - \int_0^1 \frac{\phi_m(u)}{\widehat p_j(u)}\,du, \qquad x \in [0,1].
\end{equation}
Let $M \geq 1$ be a truncation level. We introduce $K \coloneqq 1 + dM$ and for any $\mathbf x \in [0,1]^d$ the $K-$dimensional vector $\widehat{\bm\psi}_M(\mathbf x)$ which collects the constant function $1$ and these dictionary elements, ordered by coordinate and then by frequency.
The fitted constant compensates for the Lebesgue centering in
\eqref{eq:main-estimated-dictionary}; it is needed even though the
true regression function is $P-$centered.

Let $\widehat{\bm\Psi}_M\in\mathbb R^{n_1\times K}$ have rows
$\widehat{\bm\psi}_M(\mathbf X_i^1)^\top$, where
$\mathbf X_i^1$ denotes a second-subsample covariate.
Define the Gram matrices
\begin{eqnarray}
    \widehat{\bm\Gamma}_M &\coloneqq& \frac{1}{n_1} \widehat{\bm\Psi}_M^{\top}\widehat{\bm\Psi}_M, \\
    \bm\Gamma_M &\coloneqq& \mathbb E_p\!\left[ \widehat{\bm\psi}_M(\mathbf X) \widehat{\bm\psi}_M(\mathbf X)^\top \; \middle| \;\mathcal D_n^0 \right],
\end{eqnarray}
where the expectation is taken under $P$ and independently of $\mathcal D_n^0$. Define the event
\[
    \zeta_n
    =
    \left\{
        \lambda_{\min}(\widehat{\bm\Gamma}_M)\geq C_{\min} /2
    \right\}.
\]
We finally denote $\bm Y^1 \coloneqq (Y_1^1,\ldots,Y_{n_1}^1)^\top$.

\begin{definition}\label{def:thresh_estim}
    With the previous notations, we define the estimator $\widehat{f}_M$ for all $\mathbf x \in [0,1]^d$ as follows
    \begin{equation}\label{eq:main-thresholded-estimator}
    \widehat f_M(\mathbf x) \coloneqq \widehat{\boldsymbol\psi}_M(\mathbf x)^\top
\widehat{\boldsymbol\eta}_M,
\end{equation}
where $\widehat{\boldsymbol\eta}_M \coloneqq (\widehat{\bm\Psi}_M^\top\widehat{\bm\Psi}_M)^{-1} \widehat{\bm\Psi}_M^\top\bm Y^1 \ \mathbf{1}_{ \zeta_n } \in \mathbb{R}^{K}$.
\end{definition}
The inverse is evaluated only on the event $\zeta_n$, where it exists.

\paragraph{Bias-Variance decomposition.}
The analysis combines truncation and dictionary estimation into a single approximation residual. We introduce the quantity $\mu_f=\int_{[0,1]^d}f(x)\,dx $ which is used in the proof but not needed to be known and define the key function $t_M$ as follows
\begin{equation}
    t_M \coloneqq \mu_f + \sum_{j=1}^d\sum_{m=1}^M \theta_m^{(j)}\widehat\psi_m^{(j)}.
\end{equation}
Let also introduce the quantities $ r_M \coloneqq f - t_M $ and $ \bm r_M= (r_M(\mathbf X_1^1),\ldots,r_M(\mathbf X_{n_1}^1))^\top $.
The comparison function $t_M$ belongs to the fitted space and is
used only in the analysis. Its coefficients are not required to
compute the estimator. We introduce the following two quantities:
\begin{equation}
    \begin{cases}
        \mathcal B_n \coloneqq \left\| r_M -\frac1{n_1}\widehat{\bm\psi}_M^{\top} \widehat{\bm\Gamma}_M^{-1}\widehat{\bm\Psi}_M^{\top}\bm r_M \right\|^2 \mathbf{1}_{\zeta_n} ,\\
        \mathcal V_n \coloneqq \mathbb E\!\left[ \left\|\frac1{n_1}\widehat{\bm\psi}_M^{\top} \widehat{\bm\Gamma}_M^{-1}\widehat{\bm\Psi}_M^{\top}\bm\xi \right\|^2 \mathbf{1}_{\zeta_n}  \; \middle| \; \mathcal D_n\right].
    \end{cases}
\end{equation}
On the event $\zeta_n$, $\mathcal B_n$ is the conditional bias and $\mathcal V_n$ the conditional variance. Given these two terms, we have the standard bias-variance decomposition, stated in the following.

\begin{lemma}
\label{lemma:bias_variance}
For a fixed pair $(p,f)$, we have the following decompositionof the conditional risk on the high probability event
\begin{equation}
    \mathbb E_{p,f}\!\left[\|f-\widehat f_M\|^2 \mathbf{1}_{\zeta_n} \; \middle| \;\mathcal D_n \right] = \mathcal B_n + \mathcal V_n 
\end{equation}
\end{lemma}
Indeed, conditional on $\mathcal D_n$, the estimated dictionary and
the residual are fixed. Expanding the least-squares estimator gives
a residual contribution and a linear noise contribution.
The latter has conditional mean zero, so their cross term vanishes.

\paragraph{Upper bound.}
We outline the four ingredients of the proof to obtain the desired upper bound, which are detailed in Appendix. First, the most technical one is the following Lemma, which makes clearly appear the density estimation term.
\begin{lemma}[Approximation error]\label{lem:approximation}
\begin{equation}\label{eq:approximation-bound}
 \mathbb E_{p}\!\left[\|r_M\|^2\right]
 \leq2\kappa_{\max}dR^2
 \left(\frac{M^{-2\beta}}{\kappa_{\min}^2}
           +\frac{\rho_{n_0}}{\kappa_{\min}^4}\right).
\end{equation}
\end{lemma}
Then, we show in Appendix \ref{subsec:bias} the following upper bound on the bias
\begin{equation}\label{eq:bias_upper_bound}
   \mathbb{E}\left[ \mathcal B_n \; \middle| \; \mathcal D_n^0 \right] \leq \left(2+\frac{4C_{\max}}{C_{\min}}\right)\|r_M\|^2,
\end{equation}
and, we show in Appendix \ref{subsec:variance} the following upper bound on the variance
\begin{equation}\label{eq:var_upper_bound}
    \mathbb{E}\left[ \mathcal V_n \; \middle| \; \mathcal D_n^0 \right] \leq \frac{ 2 \sigma^2 C_{\max} K }{n_1 C_{\min}}.
\end{equation}
Finally, by using the Matrix-Chernoff inequality (\citet{tropp2015introduction}, recalled in our Theorem \ref{thm:tropp}), we obtain a bound on $ \mathbb{P}(\zeta_n^c \mid \mathcal D_n^0) $, that is negligible. Then, by equilibrating the bias and the variance through $M \asymp n^{ \frac{1}{2\beta + 1} }$ to obtain the desired rate, we provide the global upper bound on the class in the following theorem (proof in Appendix \ref{appendix:upper_dure}).

\begin{theorem}\label{thm:main-unknown-density-upper}
under Assumption \ref{assu:dimension}, we have the following minimax upper bound
    \begin{equation}
        \mathfrak M(n,d,\mathcal C_{\beta,\gamma}) \lesssim d \, n^{ - \frac{2 \beta}{ 2 \beta + 1 } } + d \, n^{ - \frac{2 \gamma}{ 2 \gamma + 1 } },
    \end{equation}
    where all the underlying constants behind the symbol $\lesssim$ depend only on $\beta , \gamma , L , R , \sigma , \kappa_{\min} , \kappa_{\max}$.
\end{theorem}
The bound separates the regression error from the cost of estimating the marginal densities; the slower of these two terms determines the overall rate.

\paragraph{Recovering the true components.}
The centered coordinate blocks of the estimator given in Definition \ref{def:thresh_estim} estimate $f_j-\int_0^1 f_j(t)\,dt$, rather than the true components $f_j$. Nevertheless, these components can be recovered from the same fitted coefficients.

Let $\widehat\theta_m^{(j)}$ denote the coefficient of frequency $m \in \{1, \dots, M\}$ and coordinate $j \in \{1,\dots,d\}$ of the vector $\widehat{\bm \eta}_M \in \mathbb{R}^K$ from the equation \eqref{eq:main-thresholded-estimator}.
\begin{definition}
    For all coordinate $j$, let define the debiased estimator $ \bar f_{j,M} $ of $f_j$ as follows
    \begin{equation}\label{eq:component-estimator}
    \bar f_{j,M} \coloneqq \frac{1}{\widehat p_j} \sum_{m=1}^M \widehat\theta_m^{(j)} \phi_m.
\end{equation}
\end{definition}
For a given $\widehat{\bm \eta}_M$, this estimator requires no additional fitting and natively satisfies $ \int_0^1 \bar f_{j,M}(x)\widehat p_j(x)\,dx=0 $. Thus, it is centered with respect to the estimated marginal weights. The following result bounds the total estimation error of the population-centered components.

\begin{theorem}\label{thm:component_recovery}
    under Assumption \ref{assu:dimension}, choose $M \asymp n^{ \frac{1}{2 \beta + 1} }$. Then, for all sufficiently large $n$, the risk on the estimators $\bar f_{j,M}$ achieve the following rate
    \begin{equation}
        \sum_{j=1}^d \mathbb E_{p,f}\!\left[\|\bar f_{j,M} - f_j\|^2 \right] \lesssim d\,n^{-\frac{2\beta}{(2\beta+1)}} + d\,n^{-\frac{2\gamma}{(2\gamma+1)}}.
    \end{equation}
    The implicit constants behind the symbol $\lesssim$ depend only on $\beta,\gamma,L,R,\sigma,\kappa_{\min},\kappa_{\max}$.
\end{theorem}
The proof uses the uniform lower bound on the population Gram matrix to control the fitted coefficient error by the prediction and approximation errors. Combining this control with the componentwise approximation bounds underlying Lemma~\ref{lem:approximation} yields the desired rate. As an immediate consequence, we can provide the following upper bound on the entire class
\begin{equation}\label{eq:component-risk-rate}
\boxed{
\begin{aligned}
    & \inf\limits_{\widehat f_1, \dots, \widehat f_d} \sup_{(p,f)\in\mathcal C_{\beta,\gamma}}
    \sum_{j=1}^d
    \mathbb E_{p,f}\!\left[
        \|\widehat f_{j}-f_j\|^2
    \right]\\
    &\qquad\lesssim d\,n^{-\frac{2\beta}{(2\beta+1)}} + d\,n^{-\frac{2\gamma}{(2\gamma+1)}}.
\end{aligned}}
\end{equation}
The complete proof is provided in Appendix.

\subsection{Minimax lower bound}\label{subsec:unkown_lower}

To establish the minimax lower bound, we must distinguish two regimes. The \emph{smooth} one and the \emph{rough} one. If we are in the case where $\gamma \geq \beta$, we can show that an admissible lower bound achieves the rate of $ d \, n^{ - \frac{2 \beta}{2 \beta + 1} } $. In the case where $\gamma < \beta$, the error of estimating the $d$ marginal densities should be higher than $ d \, n^{ - \frac{2 \beta}{2 \beta + 1} } $.

\paragraph{Smooth regime.}
In this paragraph, we study the case where $\gamma \geq \beta$. In the next theorem, we provide a result that holds for every $\gamma > 0$.
\begin{theorem}\label{thm:unknown-density-lower-smooth}
There exists an underlying constant which depends only on $\beta, R , \sigma$ such that
    \begin{equation}
        \mathfrak M(n,d,\mathcal C_{\beta , \gamma}) \gtrsim  d \, n^{ - \frac{2 \beta}{2 \beta + 1} }.
    \end{equation}
\end{theorem}
This result is proved in Appendix \ref{appendix:lower_smooth} and uses the fact that the uniform density $p_{\text{unif}} \equiv 1$ belongs to $\mathcal P_{\gamma , L}$ for every $\gamma > 0$. In particular, in the \emph{smooth} regime and under Assumption~\ref{assu:dimension}, we have for all sufficiently large $n$,
\begin{equation}\label{eq:unknown-density-minimax-smooth}
    \boxed{
    \mathfrak M(n,d,\mathcal C_{\beta,\gamma})
    \asymp
    d \, n^{ - \frac{2 \beta}{2 \beta + 1} }}
\end{equation}
which leads to the minimax optimality when $\gamma \geq \beta$.

\paragraph{Rough regime.}
In this paragraph, we study the case where $\gamma < \beta$. The natural rate we want to obtain is $d \, n^{ -\frac{2 \gamma}{ 2 \gamma + 1 } }$. In order to have matching bounds which tend to zero, we make the following assumption on potentially diverging dimension in this regime.
\begin{assumption}\label{assu:dimension_2}
    We suppose that $d$ can grow with $n$ at the following rate
    \begin{equation}
        d = o\!\left(n^\frac{2 \gamma}{ 2 \gamma + 1 } \right).
    \end{equation}
\end{assumption}
We establish the following theorem.
\begin{theorem}\label{thm:unknown-density-lower-rough}
    under Assumption \ref{assu:dimension_2}, we have the following minimax lower bound
    \begin{equation}
        \mathfrak M(n,d,\mathcal C_{\beta,\gamma}) \gtrsim d \, n^{ - \frac{2 \gamma}{ 2 \gamma + 1 } },
    \end{equation}
    where all the underlying constants behind the symbol $\lesssim$ depend only on $\beta , \gamma , L , R , \sigma , \kappa_{\min} , \kappa_{\max}$.
\end{theorem}
In this rough regime, we have $d \, n^{ - \frac{2 \gamma}{ 2 \gamma + 1 } } \rightarrow 0$ and $ d \, n^{ - \frac{2 \beta}{ 2 \beta + 1 } } = o( d \, n^{ - \frac{2 \gamma}{ 2 \gamma + 1 } } ) $ which leads to the minimax optimality
\begin{equation}
    \boxed{
    \mathfrak M(n,d,\mathcal C_{\beta,\gamma})
    \asymp
    d \, n^{ - \frac{2 \gamma}{2 \gamma + 1} }}
\end{equation}
The proof of this theorem is deferred to Appendix \ref{appendix:lower_rough}.

The main idea is to vary the marginal densities while keeping the weighted regression components fixed. Choose a \emph{simple} smooth univariate function $g$ with zero integral and satisfying the regular Sobolev smoothness assumption. For a sign array $\bm \omega$, construct the following joint density
\begin{equation}
    p_{\bm\omega}(\mathbf x) \coloneqq 1+\varepsilon \sin\!\left(\sum_{j=1}^d u_{{\bm\omega},j}(x_j)\right), \quad \mathbf x \in [0,1]^d,
\end{equation}
where each $u_{\omega,j}$ is a signed sum of disjoint, antisymmetric smooth bumps of width $h$ and amplitude $a h^\gamma$, with $a,\varepsilon>0$ sufficiently small. Then define the components $f_{\bm \omega , j}$ such that $p_{\bm\omega,j}f_{\bm\omega,j}=g$, so the Sobolev constraint and the centering condition hold for every alternative.

The sine coupling keeps the joint density uniformly bounded,
independently of $d$. Antisymmetry gives explicit marginals of the form $p_{\bm\omega,j} = 1 + t_h\sin u_{\bm\omega,j}$. The same symmetry yields an affine representation of $f_{\bm\omega}$ with order $d/h$ mutually orthogonal directions.

For adjacent sign arrays, the squared separation is of order
$h^{2\gamma+1}$. Writing $Q_{\bm\omega}$ for the law of one observed pair and using standard information theory results \citep{cover1991elements}, we show that the information in both the covariates and the responses satisfies
\begin{equation}
    \operatorname{KL} \bigl(Q_{\bm\omega}^{\otimes n} \| Q_{{\bm\omega}'}^{\otimes n}\bigr) \lesssim a^2 n h^{2\gamma+1}.
\end{equation}
Taking $h\asymp n^{-1/(2\gamma+1)}$ keeps neighboring alternatives statistically indistinguishable. Assouad's lemma \citep{tsybakov2009introduction} then combines their separation over the $d/h$ directions, giving the lower bound $d h^{2\gamma}=d\,n^{-2\gamma/(2\gamma+1)}$. Uniform joint-density bounds allow the argument in $L^2(\lambda_d)$ to transfer to the prediction loss under each $P_{\bm\omega}$.

\section{Discussion}\label{sec:discussion}

\paragraph{Conclusion.}
We established minimax rates for additive regression under dependent random designs with growing dimension. Our coupled smoothness framework identifies when estimating the marginal densities preserves the known-density rate and when their regularity determines the minimax rate.

\paragraph{Limitations \& future work.}
The main limitation is the assumption that the joint density is bounded above and away from zero uniformly in the dimension. Although density bounds are classical, this uniformity remains restrictive. Our guarantees also require dimension-growth conditions and do not cover sparse high-dimensional models. Relaxing the density bounds is an important direction.

A natural extension is to adapt the Riesz representation of \citet{ferrere2026generalized} to (potentially sparse) nonparametric interaction models, retaining dependent designs and diverging dimension. The aim would be to recover minimax rates with known marginals, then quantify the additional costs of estimating multivariate marginal densities and selecting active components. Kernel estimators extending \citet{raskutti2012minimax} or deep neural network estimators following \citet{bhattacharya2024deep} provide two concrete approaches
to investigate.

\newpage


\subsubsection*{Acknowledgements}
The authors thank Nicolas Bousquet (EDF R\&D) for his careful reading and valuable comments. This work was partially supported by the French \emph{Association Nationale de la Recherche et de la Technologie} (ANRT) through a CIFRE PhD project at \'Electricité de France (EDF). Fabrice Gamboa and Jean-Michel Loubes acknowledge support from the ANR-3IA Artificial and Natural Intelligence Toulouse Institute (ANITI).

\bibliography{biblio}




\clearpage
\appendix
\thispagestyle{empty}

\onecolumn
\aistatstitle{Appendix}

\startcontents[appendices]
\printcontents[appendices]{l}{1}{%
    \setcounter{tocdepth}{2}%
}
\clearpage

\section{General notations}

Let $\lambda$ denote the Lebesgue measure on $[0,1]$. We denote by $\lambda_d\coloneqq\lambda^{\otimes d}$ the product Lebesgue measure on $[0,1]^d$. Recall that for any $j \in \{1,\dots,d\}$, we denote by $P_j$ the marginal distribution of $X_j$ and by $p_j$ the corresponding density, such that $ \frac{d P_j}{d \lambda} = p_j $. We denote by $L^2(P)$ the Hilbert space of square integrable and $P-$measurable functions defined on $[0,1]^d$. For simplicity, we denote by $\| \cdot \|$ the $L^2(P)-$norm instead of $\| \cdot \|_{ L^2(P) }$ and $\| \cdot \|_2$ the Euclidean norm. When dealing with potentially infinite sequences, we denote by $\| \cdot \|_{\ell^2}$ the corresponding $\ell^2$ norm. When computing a norm under the Lebesgue measure, we will denote explicitly $\| \cdot \|_{L^2(\lambda)}$ for the univariate case and $\| \cdot \|_{L^2(\lambda_d)}$ for the $d-$variate case.

Throughout this Appendix, logarithms are natural. For probability measures $P,Q$ on the same measurable space, we use the notations
\begin{equation}
    \operatorname{TV}(P,Q):=\sup_A|P(A)-Q(A)|,
\end{equation}
for the total variation and
\begin{equation}
    \operatorname{KL}(P\|Q) \coloneqq
\begin{cases}
\displaystyle
\int\log\!\left(\frac{dP}{dQ}\right)\,dP,
&\text{if }P\ll Q,\\[1ex]
+\infty,
&\text{otherwise},
\end{cases}
\end{equation}
for the Kullback-Leibler divergence.

We recall the real trigonometric basis \citep{tsybakov2009introduction},
with $\phi_0\equiv1$ and, for every integer $m\geq1$,
\begin{equation}\label{eq:fourier}
 \phi_{2m-1}(x)=\sqrt2\cos(2\pi mx),\qquad
 \phi_{2m}(x)=\sqrt2\sin(2\pi mx),\qquad x\in[0,1].
\end{equation}
The functions $(\phi_m)_{m\geq1}$ form an orthonormal basis of the
mean-zero subspace of $L^2(\lambda)$ and satisfy
\begin{equation}\label{eq:fourier-properties}
 \int_0^1\phi_m(x)\,dx=0,\qquad
 \|\phi_m\|_{L^\infty([0,1])}\leq\sqrt2.
\end{equation}

\section{Proof of Theorem \ref{thm:riesz}}

We want to show that $\Psi$ is a Riesz sequence \citep{brezis2011functional}. Let take a finite $\ell_2$ sequence $\bm a \coloneqq \left( (a_m^{(j)})_{m \geq 1} \right)_{j=1,\dots,d}$, we define the function $f$ as
    \begin{equation}
        f \coloneqq \sum\limits_{j=1}^d \sum\limits_{m \geq 1} a_m^{(j)}  \psi_m^{(j)}.
    \end{equation}

Recall that we have for any function $u$, $\operatorname{Var}_P(u) = \inf\limits_{x \in \mathbb R} \int (u - x)^2 dP$, which leads to the fundamental inequality of the variance
\begin{equation}
    p_{\min}\operatorname{Var}_{\lambda^{\otimes d}}(u) \leq \operatorname{Var}_{P}(u) \leq p_{\max} \operatorname{Var}_{\lambda^{\otimes d}}(u).
\end{equation}

In our setting, each component $f_j$ is centered under $P_j$, so this previous inequality becomes:
\begin{equation}\label{eq:var}
    p_{\min} \sum\limits_{j=1}^d \operatorname{Var}_{\lambda}(f_j) \leq \| f \|^2 \leq p_{\max} \sum\limits_{j=1}^d \operatorname{Var}_{\lambda}(f_j)
\end{equation}

The same reasoning on each component yields
\begin{equation}
    \forall j \in \{1,\dots,d\}, \qquad p_{\min} \operatorname{Var}_{\lambda}(f_j) \leq \| f_j \|^2 \leq p_{\max} \operatorname{Var}_{\lambda}(f_j)
\end{equation}

By using this last inequality and summing over $j$, we have : 
\begin{equation}
    \frac{1}{p_{\max}} \sum\limits_{j=1}^d \| f_j \|^2 \leq \sum\limits_{j=1}^d \operatorname{Var}_{\lambda}(f_j) \leq  \frac{1}{p_{\min}} \sum\limits_{j=1}^d \| f_j \|^2
\end{equation}

By injecting this in equation \eqref{eq:var}, we obtain
\begin{equation}
    \frac{p_{\min}}{p_{\max}} \sum\limits_{j=1}^d \| f_j \|^2 \leq \| f\|^2 \leq  \frac{p_{\max}}{p_{\min}} \sum\limits_{j=1}^d \| f_j \|^2
\end{equation}

Finally, observe that by definition
\begin{equation}
    f_j = \sum_m \frac{ a_m^{(j)} \phi_m^{(j)} }{p_j},
\end{equation}
so the corresponding $\| \cdot \|^2$ norm satisfies
\begin{equation}
    \frac{1}{p_{\max}} \sum_m \left( a_m^{(j)} \right)^2 \leq \| f_j \|^2 = \int\left( \sum_m a_m^{(j)} \phi_m^{(j)} \right)^2 \frac{1}{p_j} d\lambda \leq \frac{1}{p_{\min}} \sum_m \left( a_m^{(j)} \right)^2,
\end{equation}
which leads to the desired conclusion.

\section{Proof of Theorem \ref{thm:representation}}

The Theorem~\ref{thm:riesz} shows that the collection of function $\Psi$ is a Riesz sequence of $L^2(P)$, so in order to have the result, it suffices to show that it is a Riesz basis \citep{brezis2011functional} of $\mathcal H$.

Let take a function $f = f_1 + \dots + f_d \in \mathcal H$ and set for every $j$, $g_j = p_j f_j$. The $g_j$ are centered and belong to $L^2(\lambda)$. Indeed, we have
\begin{equation}
    \int_{[0,1]} g_j(x) dx = \int_{[0,1]} f_j(x) p_j(x) dx = \mathbb{E}_{P_j}[f_j(X)] = 0,
\end{equation}
and under Assumtion \ref{assu:boundness}
\begin{equation}
   \| g_j \|_{L^2(\lambda)}^2 = \int_{[0,1]} g_j^2(x) dx = \int_{[0,1]} f^2_j(x) p_j^2(x) dx \leq p_{\max} \| f_j \|^2 < \infty. 
\end{equation}
Using that the trigonometric basis $(\phi_m)$ is a Hilbert basis of $L^2(\lambda)$ \citep{brezis2011functional,tsybakov2009introduction}, any centered function of $L^2(\lambda)$ expresses as $\sum_{m \geq 1} \alpha_m \phi_m$, which yields the following fundamental expansion of $g_j$
\begin{equation}
    g_j = \sum\limits_{m=1}^{\infty} \theta_m^{(j)} \phi_m, \qquad j=1,\dots,d,
\end{equation}
where $\theta_m^{(j)} \coloneqq \langle g_j , \phi_m\rangle_{ L^2(\lambda) }$. Now, let define the residual $R_M^{(j)}$ of order $M \geq 1$ as
\begin{equation}
    R_M^{(j)} \coloneqq f_j - \sum\limits_{m=1}^M \theta_m^{(j)} \psi_m^{(j)}.
\end{equation}
We have
\begin{align}
    \left\| R_M^{(j)} \right\|_{ L^2(P_j) }^2 &= \left\| f_j - \sum\limits_{m=1}^M \theta_m^{(j)} \psi_m^{(j)} \right\|^2 \\
    &= \int_{[0,1]} \left( f_j(x) - \sum\limits_{m=1}^M \theta_m^{(j)} \psi_m^{(j)}(x) \right)^2 p_j(x) \, dx \\
    &= \int_{[0,1]} \left( \frac{ g_j(x) - \sum\limits_{m=1}^M \theta_m^{(j)} \phi_m(x) }{p_j(x)} \right)^2 p_j(x) \,dx \\
    &\leq \frac{1}{ p_{\min} } \int_{[0,1]} \left( g_j(x) - \sum\limits_{m=1}^M \theta_m^{(j)} \phi_m(x) \right)^2 \, dx \\
    &\leq \frac{1}{ p_{\min} } \left \| g_j - \sum\limits_{m=1}^M \theta_m^{(j)} \phi_m \right\|^2_{ L^2(\lambda) }
\end{align}
This last quantity converges to $0$ when $M \rightarrow \infty$ thanks to the expansion of $g_j$ in the trigonometric basis, which ends the proof and allows to obtain the following expansion for any $f \in \mathcal H$:
\begin{equation}
\boxed{
    f = f_1 + \dots + f_d = \sum\limits_{j=1}^d \sum\limits_{m=1}^{\infty} \underbrace{\left( \int_{[0,1]} f_j(t) \phi_m(t) p_j(t) \, dt \right)}_{ \theta_m^{(j)} } \cdot \psi_m^{(j)}.}
\end{equation}

\section{Technical lemmas}

\subsection{Information theoretic lemmas}

\begin{lemma}[Elementary KL bounds, \citet{cover1991elements}]
\label{lem:e2e-kl-bounds}
For every measurable set $A$, with $t=Q(A)$ and $s=Q'(A)$,
\begin{equation}
    \operatorname{KL}(Q \| Q') \ge t\log(t/s)+(1-t)\log((1-t)/(1-s)),
\end{equation}
with the usual continuous or infinite boundary values. If $q'>0$ wherever $q>0$, then
\begin{equation}
    \operatorname{KL}(Q \| Q') \le\int\frac{(q-q')^2}{q'}.
\end{equation}
\end{lemma}

\begin{lemma}[Pinsker's inequality, \citet{cover1991elements}]
\label{lem:e2e-pinsker}
For any probability measures $Q,Q'$,
\begin{equation}
    \operatorname{TV}(Q,Q') \le\sqrt{\operatorname{KL}(Q \| Q')/2}.
\end{equation}
\end{lemma}

\begin{lemma}[Assouad's lemma for Hamming loss, \citep{tsybakov2009introduction}]
\label{lem:assouad}
Let $N\geq1$, and let $\{P_{\bm\omega}:{\bm\omega}\in\{-1,+1\}^{N}\}$ be a family of probability measures on the same measurable space. Suppose that an observation $Z$ has distribution $P_{\bm\omega}$.

Define the Hamming distance by
\begin{equation}
    \rho({\bm\omega},{\bm\omega}') \coloneqq \sum_{k=1}^{N}\mathbf 1\{\omega_k\neq\omega'_k\},
\end{equation}
and let ${\bm\omega}^{(k)}$ denote the vector obtained by reversing only the $k$-th sign of ${\bm\omega}$. If, for some $\eta\in[0,1]$,
\begin{equation}
    \max_{{\bm\omega}\in\{-1,+1\}^{N}} \max_{1\leq k\leq N} \operatorname{TV}(P_\omega,P_{{\bm\omega}^{(k)}}) \leq\eta,
\end{equation}
then
\begin{equation}
    \inf_{\widehat{\bm\omega}} \sup_{{\bm\omega}\in\{-1,+1\}^{N}} \mathbb E_{\bm\omega} \bigl[\rho(\widehat{\bm\omega}(Z),{\bm\omega})\bigr] \geq \frac{N}{2}(1-\eta),
\end{equation}
where the infimum ranges over all measurable estimators taking values in $\{-1,+1\}^{N}$.
\end{lemma}

\begin{lemma}[Orthogonal Assouad bound]
\label{lem:e2e-assouad}
Let $F_{\bm\omega} = F_0 + \sum_{k=1}^N \omega_k H_k$,
${\bm\omega}\in\{-1,1\}^N$, where the $H_k$ are orthogonal in a real Hilbert space and $\|H_k\|^2=v^2>0$.
Let $Q_{\bm\omega}$ be the observation law. If adjacent vertices satisfy
\begin{equation}
    \max_{\bm\omega,k} \operatorname{KL}(Q_{\bm\omega} \| Q_{\bm\omega^{(k)}}) \le\alpha<2,
\end{equation}
then
\begin{equation}
    \inf_{\widehat F}\sup_{\bm\omega} \mathbb E_{\bm\omega}\|\widehat F-F_{\bm\omega}\|^2 \ge\frac{Nv^2}{2}\left(1-\sqrt{\alpha/2}\right).
\end{equation}
Here ${\bm\omega}^{(k)}$ is obtained by reversing sign $k$.
\end{lemma}

\begin{proof}
For any measurable Hilbert-space-valued estimator $\widehat F$,
define
\begin{equation}
    \widehat\omega_k =
    \begin{cases}
        1, & \langle \widehat F-F_0,H_k\rangle \ge 0,\\
        -1, & \langle \widehat F-F_0,H_k\rangle < 0.
    \end{cases}
\end{equation}
By orthogonality and Bessel's inequality,
\begin{equation}
    \|\widehat F - F_{\bm\omega} \|^2 \ge \sum_{k=1}^N \frac{
 |\langle \widehat F-F_0,H_k\rangle-\omega_k v^2|^2 }{v^2} \ge v^2\sum_{k=1}^N \mathbf{1}_{\{\widehat\omega_k\ne\omega_k\}}.
\end{equation}
Let $\rho$ denote the Hamming distance. Assouad's inequality in its Kullback--Leibler form \citep{tsybakov2009introduction} gives
\begin{equation}
    \inf_{\widetilde{\bm\omega}}\sup_{\bm\omega}  \mathbb E_{\bm\omega} \rho(\widetilde{\bm\omega},{\bm\omega}) \ge \frac{N}{2}\left(1-\sqrt{\alpha/2}\right).
\end{equation}
The theorem is stated on $\{0,1\}^N$, which is equivalent
to $\{-1,1\}^N$ under a coordinatewise relabeling. Combining the last two inequalities and taking the infimum
over $\widehat F$ proves the claim.
\end{proof}

\begin{lemma}[Gaussian KL, \citet{cover1991elements}]
\label{lem:e2e-gaussian-kl}
Let $Q$ and $Q'$ be the laws of one observation with design
densities $p,p'$ and conditional distributions
$\mathcal N(f(\mathbf x),\sigma^2)$ and
$\mathcal N(f'(\mathbf x),\sigma^2)$, respectively.
When the right-hand side is finite,
\begin{equation}
    \operatorname{KL}(Q^{\otimes n} \| (Q')^{\otimes n}) = n\operatorname{KL}(P \| P') + \frac{n}{2\sigma^2}\|f-f'\|^2.
\end{equation}
\end{lemma}

\subsection{Concentration inequality on matrices}

\begin{theorem}[Matrix Chernoff, Theorem 5.1.1, \citep{tropp2015introduction}]\label{thm:tropp}
Let $\bm S_1,\ldots,\bm S_N$ be independent random Hermitian
$K\times K$ matrices such that
$0\preceq\bm S_i\preceq b_0\bm I_K$ almost surely for every $i$,
where $b_0>0$. Let define
\begin{equation}
    \mu_{\min} \coloneqq \lambda_{\min}\!\left(\sum_{i=1}^N\mathbb E[\bm S_i]\right),
\end{equation}
where the operator $\lambda_{\min}( \cdot )$ denotes the smallest eigen value. For every $0\leq\varepsilon<1$,
\begin{equation}\label{eq:tropp-chernoff}
 \mathbb P\!\left(
 \lambda_{\min}\!\left(\sum_{i=1}^N\bm S_i\right)
 \leq(1-\varepsilon)\mu_{\min}\right)
 \leq K\left(
 \frac{e^{-\varepsilon}}{(1-\varepsilon)^{1-\varepsilon}}
 \right)^{\mu_{\min}/b_0}.
\end{equation}
In particular, taking $\varepsilon=1/2$ yields
\begin{equation}\label{eq:tropp-half}
 \mathbb P\!\left(
 \lambda_{\min}\!\left(\sum_{i=1}^N\bm S_i\right)
 \leq\frac{\mu_{\min}}2\right)
 \leq K\exp\!\left(-\frac{1-\log2}{2}\frac{\mu_{\min}}{b_0}\right).
\end{equation}
\end{theorem}

\section{Proof of Theorem \ref{thm:ez_upper}}
\label{appendix:upper_facile}

Throughout this section, the design density $p$ is fixed and known.
We write $P$ for the corresponding probability distribution on
$[0,1]^d$, and $\mathcal F$ for the associated additive function
class. Expectations under the regression model with design density
$p$ and regression function $f$ are denoted by
$\mathbb E_{p,f}$.

We prove a finite-sample risk bound for the thresholded estimator
defined in \eqref{eq:oracle-thresholded-estimator}, and then choose
the truncation level $M$ to obtain the minimax upper bound.
The argument does not require independence between the coordinates
of the design.

\subsection{Notation and preliminary bounds}
\label{subsec:known-notation}

Recall that
\begin{equation}
    Y_i = f(\mathbf X_i)+\xi_i, \qquad i=1,\ldots,n,
\end{equation}
where the design vectors are independent with common distribution
$P$, and the errors are independent $\mathcal N(0,\sigma^2)$
random variables, independent of the design. For each $f\in\mathcal F$, write
\begin{equation}
    f(\mathbf x) = \sum_{j=1}^d\sum_{m\geq1} \theta_m^{(j)}\psi_m^{(j)}(x_j), \qquad \psi_m^{(j)}(x_j) =  \frac{\phi_m^{(j)}(x_j)}{p_j(x_j)}.
\end{equation}
The expansions are in $L^2(P)$. By the definition of the function class,
\begin{equation}\label{eq:known-sobolev-coefficients}
    \sum_{m\geq1}m^{2\beta} \bigl(\theta_m^{(j)}\bigr)^2 \leq R^2, \qquad j=1,\ldots,d.
\end{equation}
Recall also that the trigonometric functions are all bounded by $\sqrt{2}$.

For an integer $M\geq1$, define the truncated function
\begin{equation}
    f_M(\mathbf x) \coloneqq \sum_{j=1}^d\sum_{m=1}^M \theta_m^{(j)}\psi_m^{(j)}(x_j) = \boldsymbol\psi_M(\mathbf x)^\top \boldsymbol\theta_M,
\end{equation}
$\boldsymbol\theta_M\in\mathbb R^{dM}$ contains the first $M$ coefficients of each component, and the residual
\begin{equation}
    r_M \coloneqq f - f_M.
\end{equation}

The Riesz bounds from Theorem~\ref{thm:riesz} imply
\begin{equation}\label{eq:known-riesz-finite}
    C_{\min}^p\|\boldsymbol a\|_2^2
    \leq
    \|\boldsymbol\psi_M^\top\boldsymbol a\|^2
    =
    \boldsymbol a^\top
    \boldsymbol\Gamma_M\boldsymbol a
    \leq
    C_{\max}^p\|\boldsymbol a\|_2^2
\end{equation}
for every $\boldsymbol a\in\mathbb R^{dM}$.

\begin{remark}
    The identity $ \boldsymbol a^\top
    \boldsymbol\Gamma_M\boldsymbol a = \| \bm \psi_M^\top \bm a \|^2 $ is crucial but its proof is straightforward, indeed one has
    \begin{align}
        \bm a^\top \bm \Gamma_M \bm a &= \bm a^\top \mathbb{E}_p\left[ \bm \psi_M(\mathbf X) \bm \psi_M(\mathbf X)^\top \right] \bm a, \\ 
        &= \mathbb{E}_p\left[ \bm a^\top \bm \psi_M(\mathbf X) \bm \psi_M(\mathbf X)^\top \bm a \right], \\
        &= \mathbb{E}_p \left[ \left( \bm \psi_M(\mathbf X)^\top \bm a \right)^2 \right], \\
        &= \int (\bm \psi_M^\top \bm a )^2 \, p \, d\lambda_d, \\
        &= \| \bm \psi_M^\top \bm a \|^2.
    \end{align}
\end{remark}

In particular, applying the upper Riesz bound to the tail of
the expansion and using \eqref{eq:known-sobolev-coefficients},
we obtain
\begin{align}
    \|r_M\|^2 &\leq C_{\max}^p \sum_{j=1}^d\sum_{m>M} \bigl(\theta_m^{(j)}\bigr)^2, \\
    &\leq C_{\max}^p M^{-2\beta} \sum_{j=1}^d\sum_{m>M} m^{2\beta}\bigl(\theta_m^{(j)}\bigr)^2, \\
    &\leq C_{\max}^p dR^2M^{-2\beta}.
\end{align}\label{eq:known-approximation}
Similarly, since $m^{2\beta}\geq1$,
\begin{equation}\label{eq:known-function-energy}
    \|f\|^2 \leq C_{\max}^p \sum_{j=1}^d\sum_{m\geq1} \bigl(\theta_m^{(j)}\bigr)^2 \leq C_{\max}^p dR^2.
\end{equation}

\subsection{Thresholded least squares and the error decomposition}
\label{subsec:known-decomposition}

Recall the acceptance event and the empirical Gram matrix
\begin{equation}
    \zeta_n = \left\{ \lambda_{\min}(\widehat{\boldsymbol\Gamma}_M) \geq C_{\min}^p/2 \right\}, \qquad \widehat{\boldsymbol\Gamma}_M = \frac{1}{n} \boldsymbol\Psi_M^\top\boldsymbol\Psi_M.
\end{equation}
For clarity, recall the definition of the coefficient estimator
\begin{equation}
     \widehat{\boldsymbol\theta}_M =
    \begin{cases}
        (\boldsymbol\Psi_M^\top\boldsymbol\Psi_M)^{-1}
        \boldsymbol\Psi_M^\top\boldsymbol Y,
        & \text{on }\zeta_n,\\
        \boldsymbol 0,
        & \text{on }\zeta_n^c.
    \end{cases}
\end{equation}
Thus no inverse is evaluated on the rejection event.
On $\zeta_n$, the design matrix has full column rank, which
in particular requires $dM\leq n$. Finally, recall that the thresholded estimator $\widehat f_M$ is defined as follows
\begin{equation}
    \widehat f_M(\cdot) = \bm\psi_M(\cdot)^\top \widehat{\bm\theta}_M
\end{equation}
We introduce the random matrix
\begin{equation}
    \boldsymbol B_M \coloneqq
    \begin{cases}
        (\boldsymbol\Psi_M^\top\boldsymbol\Psi_M)^{-1}
        \boldsymbol\Psi_M^\top,
        & \text{on }\zeta_n,\\
        \boldsymbol 0,
        & \text{on }\zeta_n^c,
    \end{cases}
\end{equation}
and the vectors
\begin{equation}
    \boldsymbol r_M \coloneqq
    \bigl(r_M(\mathbf X_1),\ldots,r_M(\mathbf X_n)\bigr)^\top,
    \qquad
    \boldsymbol\xi \coloneqq (\xi_1,\ldots,\xi_n)^\top.
\end{equation}
The regression model can then be written as
\begin{equation}
    \boldsymbol Y \coloneqq \boldsymbol\Psi_M\boldsymbol\theta_M
    +\boldsymbol r_M+\boldsymbol\xi.
\end{equation}
Consequently, on $\zeta_n$, we have
\begin{equation}
    \widehat{\bm \theta}_M = \bm \theta_M + \bm B_M \bm r_M + \bm B_M \bm \xi
\end{equation}
which leads to the following decomposition of $f - \widehat f_M$
\begin{equation}\label{eq:known-error-decomposition}
    f-\widehat f_M = r_M - \boldsymbol\psi_M^\top\boldsymbol B_M\boldsymbol r_M - \boldsymbol\psi_M^\top\boldsymbol B_M\boldsymbol\xi.
\end{equation}

Let denote $\mathcal D_n$ the dataset $\mathbf X_1, \dots, \mathbf X_n$.
The event $\zeta_n$, the matrix $\boldsymbol B_M$, and the vector
$\boldsymbol r_M$ are measurable with respect to $\mathcal D_n$.
Moreover,
\[
    \mathbb E_{p,f}[
        \boldsymbol\xi \mid \mathcal D_n]
    =\boldsymbol 0,
    \qquad
    \mathbb E_{p,f}[
        \boldsymbol\xi\boldsymbol\xi^\top \mid \mathcal D_n]
    =\sigma^2\boldsymbol I_n.
\]
The cross term involving the noise therefore has conditional
expectation zero. It follows the standard bias-variance decomposition
\begin{equation}
\boxed{
     \mathbb E_{p,f}\!\left[
        \|f-\widehat f_M\|^2
        \mathbf 1_{\zeta_n}
        \; \middle| \;\mathcal D_n\right] = \left\| r_M - \boldsymbol\psi_M^\top \boldsymbol B_M\boldsymbol r_M \right\|^2 \mathbf 1_{\zeta_n} +
    \mathbb E_{p,f}\!\left[ \|\boldsymbol\psi_M^\top \boldsymbol B_M\boldsymbol\xi\|^2 \; \middle| \;\mathcal D_n
    \right]\mathbf 1_{\zeta_n}.}
    \label{eq:known-conditional-bias-variance}
\end{equation}
Recall that here, $\boldsymbol\psi_M$ is a function from $\mathbb R^d$ to $\mathbb R^{dM}$.

On $\zeta_n^c$, the estimator is zero, so
\begin{equation}\label{eq:known-rejection-identity}
\boxed{
    \mathbb E_{p,f}\!\left[ \|f-\widehat f_M\|^2 \mathbf 1_{\zeta_n^c} \right] = \|f\|^2\mathbb P_p(\zeta_n^c).}
\end{equation}
We next bound the two terms in \eqref{eq:known-conditional-bias-variance} and the rejection probability.

\subsection{Bounding the squared bias}
\label{subsec:known-bias}

On $\zeta_n$, we have $ \bm \Psi_M^\top \bm \Psi_M \succeq n C_{\min}^p / 2 $, which implies
\begin{equation}
    \boldsymbol B_M\boldsymbol B_M^\top = (\boldsymbol\Psi_M^\top\boldsymbol\Psi_M)^{-1} \preceq \frac{2}{nC_{\min}^p}\boldsymbol I_{dM}.
\end{equation}
Hence
\begin{equation}\label{eq:known-pseudoinverse-bound}
    \|\boldsymbol B_M\|_{\mathrm{op}}^2
    \leq\frac{2}{nC_{\min}^p}
    \qquad\text{on }\zeta_n.
\end{equation}

Using $\|u-v\|^2\leq2\|u\|^2+2\|v\|^2$,
\eqref{eq:known-riesz-finite}, and
\eqref{eq:known-pseudoinverse-bound}, we obtain, on $\zeta_n$,
\begin{align}
    \left\|
        r_M-\boldsymbol\psi_M^\top
        \boldsymbol B_M\boldsymbol r_M
    \right\|^2
    &\leq
    2\|r_M\|^2
    +2C_{\max}^p
        \|\boldsymbol B_M\boldsymbol r_M\|_2^2
    \notag\\
    &\leq
    2\|r_M\|^2
    +\frac{4C_{\max}^p}{nC_{\min}^p}
        \|\boldsymbol r_M\|_2^2.
    \label{eq:known-bias-pointwise}
\end{align}

Since $r_M$ is a fixed function,
\begin{equation}\label{eq:known-residual-vector}
    \mathbb E_{p,f}\!\left[
        \|\boldsymbol r_M\|_2^2
    \right]
    =
    \sum_{i=1}^n
        \mathbb E_p[r_M(\mathbf X_i)^2]
    =
    n\|r_M\|^2.
\end{equation}
Multiplying \eqref{eq:known-bias-pointwise} by
$\mathbf 1_{\zeta_n}$, taking expectations, and using
$\mathbf 1_{\zeta_n}\leq1$ yields
\begin{equation}
\boxed{
    \mathbb E_{p,f}\!\left[
        \left\|
            r_M-\boldsymbol\psi_M^\top
            \boldsymbol B_M\boldsymbol r_M
        \right\|^2
        \mathbf 1_{\zeta_n}
    \right] \leq
    \left(2+\frac{4C_{\max}^p}{C_{\min}^p}\right)
    \|r_M\|^2
    \leq
    C_{\max}^p
    \left(2+\frac{4C_{\max}^p}{C_{\min}^p}\right)
    dR^2M^{-2\beta}.}
    \label{eq:known-bias-bound}
\end{equation}

\subsection{Bounding the conditional variance}
\label{subsec:known-variance}

Conditional on $\mathcal D_n$, the matrix $\boldsymbol B_M$
is fixed. By the definition of the population Gram matrix,
\begin{align}
    &\mathbb E_{p,f}\!\left[
        \|\boldsymbol\psi_M^\top
        \boldsymbol B_M\boldsymbol\xi\|^2
        \; \middle| \;\mathcal D_n
    \right]
    \notag\\
    &\quad=
    \mathbb E_{p,f}\!\left[
        \boldsymbol\xi^\top
        \boldsymbol B_M^\top
        \boldsymbol\Gamma_M
        \boldsymbol B_M\boldsymbol\xi
        \; \middle| \;\mathcal D_n
    \right]
    \notag\\
    &\quad=
    \operatorname{Tr}\!\left(
        \boldsymbol B_M^\top
        \boldsymbol\Gamma_M
        \boldsymbol B_M
        \mathbb E_{p,f}\!\left[
            \boldsymbol\xi\boldsymbol\xi^\top
            \; \middle| \;\mathcal D_n
        \right]
    \right)
    \notag\\
    &\quad=
    \sigma^2\operatorname{Tr}\!\left(
        \boldsymbol\Gamma_M
        \boldsymbol B_M\boldsymbol B_M^\top
    \right).
    \label{eq:known-variance-identity}
\end{align}
On $\zeta_n$, this is
\[
    \frac{\sigma^2}{n}
    \operatorname{Tr}\!\left(
        \boldsymbol\Gamma_M
        \widehat{\boldsymbol\Gamma}_M^{-1}
    \right).
\]
Because
\[
    \boldsymbol\Gamma_M
    \preceq C_{\max}^p\boldsymbol I_{dM},
    \qquad
    \widehat{\boldsymbol\Gamma}_M^{-1}
    \preceq\frac{2}{C_{\min}^p}\boldsymbol I_{dM}
    \quad\text{on }\zeta_n,
\]
we obtain
\[
    \mathbb E_{p,f}\!\left[
        \|\boldsymbol\psi_M^\top
        \boldsymbol B_M\boldsymbol\xi\|^2
        \; \middle| \;\mathcal D_n
    \right]
    \leq
    \frac{2\sigma^2C_{\max}^p}{C_{\min}^p}
    \frac{dM}{n}
    \qquad\text{on }\zeta_n.
\]
Here we used the fact that $\operatorname{Tr}(\boldsymbol U\boldsymbol V)\geq0$ for positive semidefinite matrices $\boldsymbol U,\boldsymbol V$.

Since $\zeta_n$ is $\mathcal D_n$-measurable,
the tower property gives
\begin{equation}\label{eq:known-variance-bound}
    \boxed{\mathbb E_{p,f}\!\left[
        \|\boldsymbol\psi_M^\top
        \boldsymbol B_M\boldsymbol\xi\|^2
        \mathbf 1_{\zeta_n}
    \right]
    \leq
    \frac{2\sigma^2C_{\max}^p}{C_{\min}^p}
    \frac{dM}{n}.}
\end{equation}

\subsection{Concentration of the empirical Gram matrix}
\label{subsec:known-concentration}

We use the lower-tail form of the matrix Chernoff inequality (see Theorem \ref{thm:tropp}). To apply this result, write
\begin{equation}
    \widehat{\boldsymbol\Gamma}_M = \sum_{i=1}^n\boldsymbol S_i, \qquad \boldsymbol S_i  = \frac{1}{n} \boldsymbol\psi_M(\mathbf X_i) \boldsymbol\psi_M(\mathbf X_i)^\top.
\end{equation}
These matrices are independent, symmetric, and positive
semidefinite. Since a rank-one matrix $\boldsymbol v\boldsymbol v^\top$ has operator norm $\|\boldsymbol v\|_2^2$, we have
\begin{equation}
    \lambda_{\max}(\boldsymbol S_i) = \frac{1}{n} \|\boldsymbol\psi_M(\mathbf X_i)\|_2^2.
\end{equation}
The lower bound $p_j\geq p_{\min}$ and the uniform bound
on the trigonometric functions imply
\begin{align}
    \|\boldsymbol\psi_M(\mathbf x)\|_2^2
    &=
    \sum_{j=1}^d\sum_{m=1}^M
        \frac{|\phi_m^{(j)}(x_j)|^2}{p_j(x_j)^2}
    \notag\\
    &\leq\frac{2dM}{p_{\min}^2}.
    \label{eq:known-dictionary-envelope}
\end{align}
We may therefore take $ K=dM$ and $ L=\frac{2dM}{np_{\min}^2} $. Moreover,
\begin{equation}
    \sum_{i=1}^n\mathbb E_p[\boldsymbol S_i] = \boldsymbol\Gamma_M, \qquad \mu_{\min} = \lambda_{\min}(\boldsymbol\Gamma_M) \geq C_{\min}^p.
\end{equation}
Theorem~\ref{thm:tropp} gives
\begin{equation}
    \mathbb P_p\!\left( \lambda_{\min}(\widehat{\boldsymbol\Gamma}_M) \leq\frac{\mu_{\min}}{2} \right) \leq dM\exp\!\left( -\frac{1-\log2}{2} \frac{\mu_{\min}}{L} \right).
\end{equation}
Since $C_{\min}^p\leq\mu_{\min}$, we have the following bound on the probability
\begin{equation}\label{eq:known-rejection-probability}
\boxed{
    \mathbb P_p(\zeta_n^c)
    \leq
    dM\exp\!\left(
        -a_p\frac{n}{dM}
    \right),
    \qquad
    a_p
    \coloneqq
    \frac{p_{\min}^2C_{\min}^p(1-\log2)}{4}.}
\end{equation}

\subsection{Finite-sample risk bound}
\label{subsec:known-finite-sample}

Combining \eqref{eq:known-rejection-identity},
\eqref{eq:known-function-energy}, and
\eqref{eq:known-rejection-probability} yields
\begin{equation}\label{eq:known-rejection-risk}
    \mathbb E_{p,f}\!\left[
        \|f-\widehat f_M\|^2
        \mathbf 1_{\zeta_n^c}
    \right]
    \leq
    C_{\max}^p dR^2
    \left[
        dM\exp\!\left(-a_p\frac{n}{dM}\right)
    \right].
\end{equation}
On the acceptance event, we use
\eqref{eq:known-conditional-bias-variance},
\eqref{eq:known-bias-bound}, and
\eqref{eq:known-variance-bound}.
Adding the two contributions gives
\begin{align}
    \sup_{f\in\mathcal F}
    \mathbb E_{p,f}\!\left[
        \|f-\widehat f_M\|^2
    \right]
    &\leq
    C_{\max}^p
    \left(2+\frac{4C_{\max}^p}{C_{\min}^p}\right)
    dR^2M^{-2\beta}
    \notag\\
    &\quad+
    \frac{2\sigma^2C_{\max}^p}{C_{\min}^p}
    \frac{dM}{n}
    \notag\\
    &\quad+
    C_{\max}^p dR^2
    \left[
        dM\exp\!\left(-a_p\frac{n}{dM}\right)
    \right].
    \label{eq:known-finite-sample-risk}
\end{align}
All bounds are uniform over $f\in\mathcal F$. The three terms in \eqref{eq:known-finite-sample-risk} correspond to approximation, variance, and rejection and there is no regularization bias.

\subsection{Obtaining minimax upper bound}
\label{subsec:known-rate}

We typically choose $M=\left\lceil n^{1/(2\beta+1)}\right\rceil$ to match the bias and the variance. Thus the first two terms in \eqref{eq:known-finite-sample-risk} are bounded by a constant multiple of $d\,n^{ - \frac{2 \beta}{ 2 \beta + 1 } }$. Recall that under Assumption \ref{assu:dimension}, one has
\begin{equation}
    d\log n = o\!\left(n^{2\beta/(2\beta+1)}\right).
\end{equation}
Under this condition,
\begin{equation}
    \frac{n}{dM\log n}\longrightarrow\infty,
\end{equation}
and $dM\leq n$ for all sufficiently large $n$. It follows that
\begin{equation}
    dM\exp\!\left(-a_p\frac{n}{dM}\right) = o\!\left(n^{- \frac{2\beta}{2\beta + 1} }\right),
\end{equation}
so the rejection contribution is negligible relative to the target risk.

Finally, the estimator is measurable with respect to the sample
and uses only the known density and the specified class parameters.
It does not depend on the unknown regression function.
By the definition of the minimax risk,
\begin{align}
    \mathfrak M(n,d,\mathcal F) &= \inf_{\widehat f} \sup_{f\in\mathcal F} \mathbb E_{p,f}\!\left[ \|f-\widehat f\|^2 \right] \\
    &\leq \sup_{f\in\mathcal F} \mathbb E_{p,f}\!\left[ \|f-\widehat f_M\|^2 \right] \\
    &\lesssim d\,n^{- \frac{2\beta}{ 2\beta + 1 } }. \label{eq:known-minimax-upper}
\end{align}
Using the bounds on $C_{\min}^p$ and $C_{\max}^p$ supplied by
Theorem~\ref{thm:riesz}, the implicit constant depends only on
$\beta,R,\sigma,p_{\min},p_{\max}$, and not on $n$, $d$, or $f$.

\section{Proof of Theorem \ref{thm:ez_lower}}\label{appendix:lower_facile}

Throughout this section, the design density $p$ is fixed and known.
We write $P$ for the corresponding distribution on $[0,1]^d$,
and $\mathcal F$ for the associated additive function class.
We assume $\beta>0$, $R>0$, and $\sigma>0$.

We prove that
\[
    \mathfrak M(n,d,\mathcal F)
    \coloneqq
    \inf_{\widehat f}
    \sup_{f\in\mathcal F}
    \mathbb E_{p,f}\!\left[
        \|\widehat f-f\|^2
    \right]
    \gtrsim
    d\,n^{-\frac{2\beta}{2\beta+1}}.
\]
The infimum includes all estimators based on the observations
and the known density $p$.

Our argument uses the Riesz inequalities from
Theorem~\ref{thm:riesz}. The constants are independent of the number of retained frequencies. Their uniform control in terms of $p_{\min}$ and $p_{\max}$ will give the claimed dependence of the final constant.

\subsection{Construction of a finite subfamily}
\label{subsec:known-lower-construction}

For $n\geq1$, set $ M=\left\lceil n^{ \frac{1}{2\beta + 1} }\right\rceil $ and $ a_M=a_0M^{-\beta-1/2} $ where
\begin{equation}\label{eq:known-lower-amplitude}
    a_0^2
    =
    \min\left\{
        2^{-2\beta}R^2,\,
        \frac{\sigma^2}{16C_{\max}^p}
    \right\}.
\end{equation}
In particular, $a_0>0$ and does not depend on $n$ or $M$. Let define
\begin{equation}
    \mathcal I_M = \{(j,m):1\leq j\leq d,\ M+1\leq m\leq2M\}, \qquad |\mathcal I_M|=dM.
\end{equation}
For each sign vector
$\boldsymbol\omega\in\{-1,+1\}^{\mathcal I_M}$, define
\begin{equation}
    f_{\boldsymbol\omega}(\mathbf x) = \sum_{j=1}^d f_{\boldsymbol\omega,j}(x_j), \qquad f_{\boldsymbol\omega,j}(x_j) = a_M\sum_{m=M+1}^{2M} \omega_{jm}\psi_m^{(j)}(x_j).
\end{equation}
Observe that the reference-basis coefficients of $p_jf_{\boldsymbol\omega,j}$ equal $a_M\omega_{jm}$ for
$M+1\leq m\leq2M$ and vanish otherwise. Hence
\begin{align}
    \sum_{m=M+1}^{2M}
        m^{2\beta}(a_M\omega_{jm})^2
    &\leq
    M(2M)^{2\beta}a_M^2
    \notag\\
    &=2^{2\beta}a_0^2
    \notag\\
    &\leq R^2.
    \label{eq:known-lower-membership}
\end{align}
It follows that
\begin{equation}
    \left\{ f_{\boldsymbol\omega}: \boldsymbol\omega\in\{-1,+1\}^{\mathcal I_M} \right\} \subseteq\mathcal F.
\end{equation}

For later use, define the following Hamming distance
\begin{equation}
    \rho(\boldsymbol\omega,\boldsymbol\omega') \coloneqq \sum_{(j,m)\in\mathcal I_M} \mathbf 1_{\{\omega_{jm}\neq\omega'_{jm}\}}.
\end{equation}
The lower Riesz inequality immediately gives
\begin{equation}\label{eq:known-lower-separation}
    \|f_{\boldsymbol\omega}
      -f_{\boldsymbol\omega'}\|^2
    \geq
    4C_{\min}^p a_M^2
    \rho(\boldsymbol\omega,\boldsymbol\omega').
\end{equation}

\subsection{Controlling neighboring observation laws}
\label{subsec:known-lower-testing}

Let $Q_{\boldsymbol\omega}$ denote the joint law of the
$n$ observations under the design distribution $P$ and regression
function $f_{\boldsymbol\omega}$.
Write $\mathbb E_{\boldsymbol\omega}$ for expectation under
this law.

For $(j,m)\in\mathcal I_M$, let
$\boldsymbol\omega^{(j,m)}$ be obtained from
$\boldsymbol\omega$ by reversing only the sign $\omega_{jm}$.
Then
\begin{equation}
    f_{\boldsymbol\omega} - f_{\boldsymbol\omega^{(j,m)}} = 2a_M\omega_{jm}\psi_m^{(j)}.
\end{equation}
By the upper Riesz inequality,
\begin{equation}\label{eq:known-lower-neighbor-distance}
    \|f_{\boldsymbol\omega}
      -f_{\boldsymbol\omega^{(j,m)}}\|^2
    \leq4C_{\max}^p a_M^2.
\end{equation}
The design distribution is identical at every vertex of the hypercube. The Gaussian regression identity in Lemma~\ref{lem:e2e-gaussian-kl} therefore yields
\begin{align}
    \operatorname{KL}\!\left(
        Q_{\boldsymbol\omega} \|
        Q_{\boldsymbol\omega^{(j,m)}}
    \right)
    &=n \times 0 + \frac{n}{2\sigma^2}
    \|f_{\boldsymbol\omega}
      -f_{\boldsymbol\omega^{(j,m)}}\|^2
    \notag\\
    &\leq
    \frac{2nC_{\max}^p a_M^2}{\sigma^2}
    \notag\\
    &=
    \frac{2C_{\max}^p a_0^2}{\sigma^2}
    \frac{n}{M^{2\beta+1}}
    \notag\\
    &\leq\frac18.
    \label{eq:known-lower-neighbor-kl}
\end{align}
The last inequality follows from $M^{2\beta+1}\geq n$ and \eqref{eq:known-lower-amplitude}.

Pinsker's inequality (Lemma~\ref{lem:e2e-pinsker}) now gives
\begin{equation}\label{eq:known-lower-neighbor-tv}
    \max_{\boldsymbol\omega}
    \max_{(j,m)\in\mathcal I_M}
    \operatorname{TV}\!\left(
        Q_{\boldsymbol\omega},
        Q_{\boldsymbol\omega^{(j,m)}}
    \right)
    \leq
    \sqrt{\frac12\cdot\frac18}
    =\frac14.
\end{equation}

\subsection{Reduction from prediction loss to Hamming loss}
\label{subsec:known-lower-reduction}

Fix an arbitrary estimator $\widehat f$.
If its maximum risk over the finite subfamily is infinite,
the desired lower bound is immediate.
We may therefore assume that $\widehat f$ takes values in
$L^2(P)$ almost surely under every $Q_{\boldsymbol\omega}$. Consider the finite-dimensional vector subspace
\begin{equation}
    \mathcal V_M
    =
    \operatorname{span}\left\{
        \psi_m^{(j)}:
        (j,m)\in\mathcal I_M
    \right\},
\end{equation}
and let $\Pi_M$ be the orthogonal projection onto $\mathcal V_M$ in $L^2(P)$. The lower Riesz bound guarantees linear independence of the
functions defining $\mathcal V_M$. Consequently, there is a
unique coefficient vector $\widehat{\boldsymbol b} =(\widehat b_{jm})_{(j,m)\in\mathcal I_M}$ such that
\begin{equation}
    \Pi_M\widehat f = \sum_{(j,m)\in\mathcal I_M} \widehat b_{jm}\psi_m^{(j)}.
\end{equation}
The coefficient map is continuous on $\mathcal V_M$, so these
coefficients are measurable functions of the observations. Define the induced sign estimator by
\begin{equation}
    \widehat\omega_{jm} =
    \begin{cases}
        +1,& \widehat b_{jm}\geq0,\\
        -1,& \widehat b_{jm}<0.
    \end{cases}
\end{equation}
Because $f_{\boldsymbol\omega}\in\mathcal V_M$, orthogonal projection gives
\begin{align}
    \|\widehat f-f_{\boldsymbol\omega}\|^2
    &= \|\widehat f-\Pi_M\widehat f\|^2 + \|\Pi_M\widehat f-f_{\boldsymbol\omega}\|^2 \notag\\
    &\geq \|\Pi_M\widehat f-f_{\boldsymbol\omega}\|^2 \notag\\
    &\geq C_{\min}^p \sum_{(j,m)\in\mathcal I_M} (\widehat b_{jm}-a_M\omega_{jm})^2.
    \label{eq:known-lower-hamming-reduction}
\end{align}
Now, let suppose that $ \widehat\omega_{jm}\neq\omega_{jm} $, because $a_M >0$, we have
\begin{equation}
    \begin{cases}
        \omega_{jm} = +1 \implies \widehat{\omega}_{jm} = -1 \implies \widehat{b}_{jm} < 0 \implies | \widehat{b}_{jm} - a_M \omega_{jm}| = a_M - \widehat{b}_{jm} \geq a_M \\
        \omega_{jm} = -1 \implies \widehat{\omega}_{jm} = +1 \implies \widehat{b}_{jm} \geq 0 \implies | \widehat{b}_{jm} - a_M \omega_{jm}| = \widehat b_{jm} + a_M \geq a_M
    \end{cases}
\end{equation}
We finally can write
\begin{equation}
    \left( \widehat b_{jm} - a_M \omega_{jm} \right)^2 \geq a_M^2 \mathbf{1}_{ \{ \widehat{\omega}_{jm} \neq \omega_{jm} \} },
\end{equation}
which leads to the last inequality
\begin{equation}
    \|\widehat f-f_{\boldsymbol\omega}\|^2\geq C_{\min}^p a_M^2
    \rho(\widehat{\boldsymbol\omega},\boldsymbol\omega).
\end{equation}

\subsection{Application of Assouad's lemma}
\label{subsec:known-lower-assouad}

By \eqref{eq:known-lower-neighbor-tv},
Lemma~\ref{lem:assouad} applies with $N=dM$ and $\eta=1/4$. Therefore every measurable sign estimator satisfies
\begin{equation}
    \sup_{\boldsymbol\omega} \mathbb E_{\boldsymbol\omega}\!\left[ \rho(\widehat{\boldsymbol\omega},\boldsymbol\omega) \right] \geq \frac{dM}{2}\left(1-\frac14\right) = \frac38dM.
\end{equation}
Since the finite subfamily is contained in $\mathcal F$,
\eqref{eq:known-lower-hamming-reduction} implies
\begin{align}
    \sup_{f\in\mathcal F}
    \mathbb E_{p,f}\!\left[
        \|\widehat f-f\|^2
    \right]
    &\geq
    \sup_{\boldsymbol\omega}
    \mathbb E_{\boldsymbol\omega}\!\left[
        \|\widehat f-f_{\boldsymbol\omega}\|^2
    \right]
    \notag\\
    &\geq
    C_{\min}^p a_M^2
    \sup_{\boldsymbol\omega}
    \mathbb E_{\boldsymbol\omega}\!\left[
        \rho(\widehat{\boldsymbol\omega},\boldsymbol\omega)
    \right]
    \notag\\
    &\geq
    \frac38 C_{\min}^p a_M^2dM
    \notag\\
    &\geq \frac38 C_{\min}^p a_0^2dM^{-2\beta}.
    \label{eq:known-lower-before-rate}
\end{align}
This holds for every estimator $\widehat f$. Taking the infimum over estimators yields
\begin{equation}
    \mathfrak M(n,d,\mathcal F) \geq \frac38 C_{\min}^p a_0^2dM^{-2\beta}.
\end{equation}
Finally, by recalling that for $n\geq1, M = \left\lceil n^{\frac{1}{2\beta + 1}}\right\rceil$, we have the desired result
\begin{equation}
    \mathfrak M(n,d,\mathcal F) \gtrsim d\,n^{-\frac{2\beta}{2\beta+1}},
\end{equation}
where the positive underlying constant is equal to $ \frac{3C_{\min}^p}{2^{2\beta+3}} \min\left\{ 2^{-2\beta}R^2,\, \frac{\sigma^2}{16C_{\max}^p}\right\} $ and does not depend on $n$.

\begin{remark}
    The lower-bound argument itself imposes no restriction on the growth of $d$, provided the Riesz constants remain uniformly controlled. The dimension condition used for the upper bound is needed to establish the matching upper rate.
\end{remark}

\section{Proof of Theorem \ref{thm:main-unknown-density-upper}}\label{sec:proof}
\label{appendix:upper_dure}

\subsection{Notations}\label{subsec:notations}

We recall that, throughout the class $\mathcal C_{\beta,\gamma}$,
the joint densities satisfy
\begin{equation}\label{eq:density-bounds}
 0< \kappa_{\min} \leq p \leq \kappa_{\max}<\infty.
\end{equation}
Where constants $\kappa_{\min}$ and $\kappa_{\max}$ are parameters defining the class and satisfying 
\begin{equation}
    \kappa_{\min} < 1 < \kappa_{\max}.
\end{equation}
As a direct consequence, the class is nonempty, since the uniform distribution belongs to the admissible distributions. Integrating \eqref{eq:density-bounds} over all but one coordinate gives $ \kappa_{\min} \leq p_j\leq \kappa_{\max}$ for every $j$. No independence between the coordinates of a covariate is assumed.

Fix a pair $(p,f)\in\mathcal C_{\beta,\gamma}$. The regression function is additive and centered componentwise. We use the representations
\begin{equation}\label{eq:component-representation}
 f_j=\sum_{m\geq1}\theta_m^{(j)}\psi_m^{(j)},\qquad
 g_j\coloneqq p_jf_j
     =\sum_{m\geq1}\theta_m^{(j)}\phi_m,
 \qquad 1\leq j\leq d.
\end{equation}
The Sobolev condition implies
\begin{equation}\label{eq:sobolev-class}
 \sum_{m\geq1}m^{2\beta}\bigl(\theta_m^{(j)}\bigr)^2\leq R^2,
 \qquad 1\leq j\leq d,
\end{equation}
where $\beta>0$ and $R>0$ are fixed.
Thus Sobolev regularity is understood in the periodic Fourier sense
specified by \eqref{eq:sobolev-class}. The expansions in
\eqref{eq:component-representation} hold in $L^2(\lambda)$ for $g_j$
and in $L^2(p_j\,d\lambda)$ for $f_j$.
No assumption $\beta>1/2$ is made.

The marginal densities belong are all $(\gamma ,L)-$Hölder (see Definition \ref{def:Holder}) and uniformly bounded between the class parameters $\kappa_{\min}$ and $\kappa_{\max}$. The notation $\mathcal C_{\beta,\gamma}$ suppresses the fixed parameters $R,L,\kappa_{\min},\kappa_{\max}$. Both $p$ and $f$ may vary in this class, and $d$ may increase with $n$ at the rate of Assumption \ref{assu:dimension}.

We observe an i.i.d. sample
$(\mathbf X_1,Y_1),\ldots,(\mathbf X_n,Y_n)$ satisfying
\begin{equation}
 Y_i=f(\mathbf X_i)+\xi_i,\qquad
 \mathbf X_i\sim P,\qquad
 \xi_i\sim\mathcal N(0,\sigma^2),\qquad 1\leq i\leq n,
\end{equation}
where the noises are independent of all covariates and $\sigma^2>0$
is fixed. For $n\geq4$, let
\begin{equation}
     n_0\coloneqq\lfloor n/2\rfloor,\qquad n_1\coloneqq n-n_0.
\end{equation}
We split the covariates into
\begin{align}
 \mathcal D_n^0
 &\coloneqq(\mathbf X_1^0,\ldots,\mathbf X_{n_0}^0)
       =(\mathbf X_1,\ldots,\mathbf X_{n_0}),\\
 \mathcal D_n^1
 &\coloneqq(\mathbf X_1^1,\ldots,\mathbf X_{n_1}^1)
       =(\mathbf X_{n_0+1},\ldots,\mathbf X_n).
\end{align}
These samples are independent and have sizes of order $n$. The first sample is used to estimate the marginal densities; the second, together with its responses, is used to estimate $f$. We write $Y_i^1\coloneqq Y_{n_0+i}$ and $\xi_i^1\coloneqq\xi_{n_0+i}$. Conditioning on a dataset means conditioning on the sigma-field generated by its covariates. In particular, $\mathcal D_n\coloneqq(\mathcal D_n^0,\mathcal D_n^1)$ contains the two designs, but not the responses.

Let $\widehat p_1,\ldots,\widehat p_d$ be jointly measurable estimators
constructed from $\mathcal D_n^0$ alone. We will require
\begin{equation}\label{eq:estimated-density-bounds}
 \kappa_{\min}\leq\widehat p_j(x)\leq \kappa_{\max}
 \quad\text{for every }x\in[0,1],\ 1\leq j\leq d,
 \quad\text{almost surely}.
\end{equation}
The marginal-density estimators constructed satisfy
\begin{equation}\label{eq:density-risk-assumption}
\max_{1\leq j\leq d}\sup_{x\in[0,1]}
\mathbb E\!\left[\left(
\widehat p_j(x)-p_j(x)\right)^2
\right]
\leq \rho_{n_0},
\end{equation}
where $\rho_{n_0}$ is a deterministic bound depending only
on the sample size $n_0$ and the fixed parameters of the class,
and not on the particular pair $(p,f)$ or the dimension $d$.
Subsection~\ref{subsec:density-construction} constructs estimators satisfying
these conditions with $\rho_{n_0}\leq C_{\gamma,L,\kappa_{\max}}
n_0^{-2\gamma/(2\gamma+1)}$. The estimators of different marginal densities may be dependent.

For these estimators and an integer $M\geq1$, define the estimated
dictionary elements centered with respect to Lebesgue measure:
\begin{equation}\label{eq:estimated-dictionary}
 \widehat\psi_m^{(j)}(x_j)
 \coloneqq\frac{\phi_m(x_j)}{\widehat p_j(x_j)}
       -\int_0^1\frac{\phi_m(t)}{\widehat p_j(t)}\,dt,
 \qquad 1\leq j\leq d,\quad1\leq m\leq M.
\end{equation}
For $\mathbf x\in[0,1]^d$, let
\begin{equation}\label{eq:dictionary-vector}
 \widehat{\bm\psi}(\mathbf x)
 \coloneqq
 \left(1,\widehat\psi_1^{(1)}(x_1),\ldots,
 \widehat\psi_M^{(1)}(x_1),\ldots,
 \widehat\psi_1^{(d)}(x_d),\ldots,
 \widehat\psi_M^{(d)}(x_d)\right)^\top
 \in\mathbb R^K,
 \qquad K\coloneqq1+dM.
\end{equation}
The order of all coefficient vectors below matches this order of the
dictionary elements. The design matrix
$\widehat{\bm\Psi}\in\mathbb R^{n_1\times K}$ has $i$th row
$\widehat{\bm\psi}(\mathbf X_i^1)^\top$. Its empirical Gram matrix and
conditional population counterpart are
\begin{equation}\label{eq:gram-matrices}
 \widehat{\bm\Gamma}
 \coloneqq\frac1{n_1}\widehat{\bm\Psi}^\top\widehat{\bm\Psi},
 \qquad
 \bm\Gamma
 \coloneqq\mathbb E\!\left[
 \widehat{\bm\psi}(\mathbf X)\widehat{\bm\psi}(\mathbf X)^\top
 \; \middle| \;\mathcal D_n^0\right],
\end{equation}
where $\mathbf X\sim P$ is independent of $\mathcal D_n^0$. 

For simplicity, we omit the dependence in $M$ all previous notations. 

We define 4 constants which depend only of $\kappa_{\min}$ and $\kappa_{\max}$:
\begin{equation}\label{eq:constants}
 \begin{aligned}
 C_{\min}&\coloneqq\frac{\kappa_{\min}}{\kappa_{\max}^2},
 & C_{\max}&\coloneqq\frac{\kappa_{\max}}{\kappa_{\min}^2},\\
 c_1&\coloneqq\frac{C_{\min}}{2},
 & c_2&\coloneqq
 \frac{(1-\log2)\kappa_{\min}^3}{16\kappa_{\max}^2}>0.
 \end{aligned}
\end{equation}
For a matrix, $\|\cdot\|_{\mathrm{op}}$ denotes its matrix operator norm, and $\preceq$ denotes the Loewner order on symmetric matrices.

\subsection{The estimator}\label{subsec:estimator}

Fix $M\geq1$ and marginal density estimators satisfying
\eqref{eq:estimated-density-bounds}--\eqref{eq:density-risk-assumption}.
Define the event
\begin{equation}
 \zeta_n\coloneqq
 \{\lambda_{\min}(\widehat{\bm\Gamma})\geq c_1\},
\end{equation}
where $\lambda_{\min}$ denotes the smallest eigenvalue. Let
\begin{equation}
    \bm Y^1\coloneqq(Y_1^1,\ldots,Y_{n_1}^1)^\top.
\end{equation}
The thresholded least-squares estimator is
\begin{equation}
 \widehat f_M(\mathbf x)\coloneqq
 \begin{cases}
 \widehat{\bm\psi}(\mathbf x)^\top
 (\widehat{\bm\Psi}^\top\widehat{\bm\Psi})^{-1}
 \widehat{\bm\Psi}^\top\bm Y^1,
 &\text{on }\zeta_n,\\[2mm]
 0,&\text{on }\zeta_n^c.
 \end{cases}
\end{equation}

\begin{remark}[Well-definedness and the fitted constant]
On $\zeta_n$,
\begin{equation}
    \widehat{\bm\Psi}^\top\widehat{\bm\Psi}
 =n_1\widehat{\bm\Gamma}\succeq n_1 c_1 \bm I_K,
\end{equation}
so the inverse in the estimator exists. On the complement,
no inverse is evaluated. In particular, $\zeta_n$ is empty if $K>n_1$.
The estimator uses only the observed data and the known density bounds;
it has no ridge parameter.

The fitted space is
\begin{equation}
    \operatorname{span}\left(
 1,\left\{\frac{\phi_m(x_j)}{\widehat p_j(x_j)}
             :1\leq j\leq d,\ 1\leq m\leq M\right\}\right).
\end{equation}
The constant allows the estimator to account for the means removed
in \eqref{eq:estimated-dictionary}. It does not introduce a nonzero
intercept into the true model \eqref{eq:additive-model}. Centering under
$P$ does not, in general, imply centering under $\lambda_d$.
The integrals in \eqref{eq:estimated-dictionary} are known functionals
of the estimated densities. The theoretical estimator uses these exact
integrals; no closed-form antiderivative is assumed.
\end{remark}

\subsection{Proof of Lemma \ref{lemma:bias_variance}}\label{subsec:decomposition}

For $M\geq1$, introduce the truncated numerators and the Lebesgue mean
\begin{equation}
    h_j\coloneqq\sum_{m=1}^M\theta_m^{(j)}\phi_m, \qquad \mu_f\coloneqq\int_{[0,1]^d}f\,d\lambda_d.
\end{equation}
Define the comparison function
\begin{equation}\label{eq:oracle}
 t_M(\mathbf x)\coloneqq\mu_f+
 \sum_{j=1}^d\left\{
 \frac{h_j(x_j)}{\widehat p_j(x_j)}
 -\int_0^1\frac{h_j(t)}{\widehat p_j(t)}\,dt\right\}.
\end{equation}
This function incorporates both truncation and the estimated dictionary.
It is generally random through $\mathcal D_n^0$.
Although $\mathbb E_P f=0$, the coefficient $\mu_f$ need not vanish
when $P$ is nonuniform. Neither $\mu_f$ nor the coefficients of the
true function are used by the estimator.

Let define
\begin{equation}
    \bm\eta^\star\coloneqq
 \left(\mu_f,\theta_1^{(1)},\ldots,\theta_M^{(1)},\ldots,
 \theta_1^{(d)},\ldots,\theta_M^{(d)}\right)^\top\in\mathbb R^K.
\end{equation}
Then $t_M(\cdot)=\widehat{\bm\psi}(\cdot)^\top\bm\eta^\star$. Indeed, by the definition of the centered dictionary and linearity of the integral,
\begin{align*}
 \widehat{\bm\psi}(\mathbf x)^\top\bm\eta^\star
 &=\mu_f+\sum_{j=1}^d\sum_{m=1}^M
          \theta_m^{(j)}\widehat\psi_m^{(j)}(x_j)\\
 &=\mu_f+\sum_{j=1}^d\left\{
 \frac{\sum_{m=1}^M\theta_m^{(j)}\phi_m(x_j)}{\widehat p_j(x_j)}
 -\int_0^1
 \frac{\sum_{m=1}^M\theta_m^{(j)}\phi_m(t)}{\widehat p_j(t)}\,dt
 \right\}\\
 &=t_M(\mathbf x).
\end{align*}

Set
\begin{equation}\label{eq:residual-definitions}
 \begin{aligned}
 r&\coloneqq f-t_M,\\
 \bm r&\coloneqq
 (r(\mathbf X_1^1),\ldots,r(\mathbf X_{n_1}^1))^\top,\\
 \bm\xi&\coloneqq(\xi_1^1,\ldots,\xi_{n_1}^1)^\top.
 \end{aligned}
\end{equation}
The responses satisfy the exact identity
\begin{equation}\label{eq:response-decomposition}
 \bm Y^1=\widehat{\bm\Psi}\bm\eta^\star+\bm r+\bm\xi.
\end{equation}
Consequently, on $\zeta_n$,
\begin{equation}\label{eq:estimator-decomposition}
 \widehat f_M(\mathbf x) =t_M(\mathbf x) + \widehat{\bm\psi}(\mathbf x)^\top (\widehat{\bm\Psi}^\top\widehat{\bm\Psi})^{-1} \widehat{\bm\Psi}^\top\bm r + \widehat{\bm\psi}(\mathbf x)^\top (\widehat{\bm\Psi}^\top\widehat{\bm\Psi})^{-1} \widehat{\bm\Psi}^\top\bm\xi.
\end{equation}
Equivalently,
\begin{equation}\label{eq:error-decomposition}
 f(\mathbf x)-\widehat f_M(\mathbf x) = r(\mathbf x)
 -\frac1{n_1}\widehat{\bm\psi}(\mathbf x)^\top
       \widehat{\bm\Gamma}^{-1}\widehat{\bm\Psi}^\top\bm r - \frac1{n_1}\widehat{\bm\psi}(\mathbf x)^\top \widehat{\bm\Gamma}^{-1}\widehat{\bm\Psi}^\top\bm\xi.
\end{equation}
Conditional on the two designs $\mathcal D_n$, the dictionary, $r$,
$\bm r$, and $\zeta_n$ are fixed, whereas
\begin{equation}\label{eq:conditional-noise}
 \mathbb E[\bm\xi \mid \mathcal D_n]=0,
 \qquad
 \mathbb E[\bm\xi\bm\xi^\top \mid \mathcal D_n]
   =\sigma^2\bm I_{n_1}.
\end{equation}
The cross term involving the noise therefore has conditional expectation
zero. On $\zeta_n$, we obtain the bias-variance decomposition
\begin{equation}\label{eq:conditional-bias-variance}
\boxed{
\mathbb E\!\left[\|f-\widehat f_M\|^2 \mathbf{1}_{\zeta_n} \; \middle| \;\mathcal D_n\right] = \underbrace{\left\| r-\frac1{n_1}\widehat{\bm\psi}^{\top} \widehat{\bm\Gamma}^{-1}\widehat{\bm\Psi}^{\top}\bm r \right\|^2 \mathbf{1}_{\zeta_n} }_{\text{conditional squared bias}} + \underbrace{\mathbb E\!\left[ \left\|\frac1{n_1}\widehat{\bm\psi}^{\top} \widehat{\bm\Gamma}^{-1}\widehat{\bm\Psi}^{\top}\bm\xi \right\|^2 \mathbf{1}_{\zeta_n} \; \middle| \;\mathcal D_n\right] }_{\text{conditional variance}}.}
\end{equation}
Here $\widehat{\bm\psi}$ denotes the vector-valued function of the test
covariate inside the $L^2(P)$ norm. On $\zeta_n^c$, $\widehat f_M = 0$ so the conditional risk is just given by
\begin{equation}
\boxed{
    \mathbb E\!\left[\|f-\widehat f_M\|^2 \mathbf{1}_{\zeta_n^c} \; \middle| \; \mathcal D_n\right] = \left\| f \right\|^2 \mathbf{1}_{\zeta_n^c}.}
\end{equation}

\subsection{Uniform geometry of the estimated dictionary}\label{subsec:geometry}

We first introduce the standard lemma which bridges the population Gram matrix $\bm \Gamma$ and the function $\widehat{\bm \psi}$.

\begin{lemma}
    Let $\bm \eta \in \mathbb R^K$, we have the following identity
    \begin{equation}
        \| \widehat {\bm \psi} \bm \eta \|^2 = \bm \eta^\top \bm \Gamma \bm \eta.
    \end{equation}
\end{lemma}

\begin{proof}
    It suffices to start from the definition of $\Gamma$ and compute $\bm \eta^\top \bm \Gamma \bm \eta$, indeed we have
    \begin{align}
        \bm \eta^\top \bm \Gamma \bm \eta &= \bm \eta^\top \mathbb E\!\left[
 \widehat{\bm\psi}(\mathbf X)\widehat{\bm\psi}(\mathbf X)^\top
 \; \middle| \;\mathcal D_n^0\right] \bm \eta \\
 &= \mathbb{E}\left[ \bm \eta^\top \widehat{\bm\psi}(\mathbf X)\widehat{\bm\psi}(\mathbf X)^\top \bm \eta \; \middle| \;\mathcal D_n^0 \right] \\
 &= \mathbb{E}\left[ \left( \bm \eta^\top \widehat{\bm\psi}(\mathbf X) \right)^2 \; \middle| \;\mathcal D_n^0 \right] \\
 &= \int \left( \bm \eta^\top \widehat{\bm\psi}(\mathbf x) \right)^2 p(\mathbf x) d\mathbf x \\
 &= \| \widehat {\bm \psi} \bm \eta \|^2.
    \end{align}
\end{proof}

\begin{lemma}[Population Gram matrix and row norm]\label{lem:geometry}
For every realization of the first sample satisfying
\eqref{eq:estimated-density-bounds},
\begin{equation}\label{eq:population-gram-bounds}
 C_{\min}\bm I_K\preceq\bm\Gamma \preceq C_{\max}\bm I_K.
\end{equation}
Moreover,
\begin{equation}\label{eq:row-norm}
 \sup_{\mathbf x\in[0,1]^d}
 \|\widehat{\bm\psi}(\mathbf x)\|_2^2
 \leq1+\frac{8dM}{\kappa_{\min}^2}
 \leq\frac{8K}{\kappa_{\min}^2}.
\end{equation}
These bounds require no event on which the estimated densities are close
to the true densities.
\end{lemma}
\begin{proof}
Fix $\mathcal D_n^0$. For $\bm u=(u_1,\ldots,u_M)^\top$, set
\begin{equation}
    h=\sum_{m=1}^M u_m\phi_m,
 \qquad
 v=\frac h{\widehat p_j}
             -\int_0^1\frac{h(t)}{\widehat p_j(t)}\,dt.
\end{equation}
Since $\int h\,d\lambda=0$ and $\widehat p_j\leq \kappa_{\max}$,
\begin{equation}
    \langle h,v\rangle_{L^2(\lambda)}
 =\int_0^1\frac{h(t)^2}{\widehat p_j(t)}\,dt
 \geq\frac1{\kappa_{\max}}\|h\|_{L^2(\lambda)}^2.
\end{equation}
Cauchy--Schwarz yields
$\|v\|_{L^2(\lambda)}\geq
\|h\|_{L^2(\lambda)}/\kappa_{\max}$; the statement is immediate if
$h=0$. Conversely, subtracting the Lebesgue mean is an orthogonal
projection in $L^2(\lambda)$, so
\begin{equation}
    \|v\|_{L^2(\lambda)}
 \leq\|h/\widehat p_j\|_{L^2(\lambda)}
 \leq\frac1{\kappa_{\min}}\|h\|_{L^2(\lambda)}.
\end{equation}
By orthonormality of the $\phi_m$, we have therefore proved
\begin{equation}\label{eq:block-geometry}
 \frac1{\kappa_{\max}^2}\|\bm u\|_2^2
 \leq\|v\|_{L^2(\lambda)}^2
 \leq\frac1{\kappa_{\min}^2}\|\bm u\|_2^2.
\end{equation}

For any $\bm\eta\in\mathbb R^K$, let compute $ \widehat{\bm\psi}^\top\bm\eta $:
\begin{align}
    \widehat{\bm\psi}^\top\bm\eta &= \eta_0 + \sum\limits_{j=1}^d \sum\limits_{m=1}^M \eta_m^{(j)} \, \widehat\psi_m^{(j)} \\
    &= \eta_0 + \sum\limits_{j=1}^d \underbrace{\sum\limits_{m=1}^M \eta_{m}^{(j)} \, \left( \frac{ \phi_m }{ \widehat{p}_j } - \int_{[0,1]} \frac{\phi_m(t)}{\widehat{p}_j(t)} \, dt \right)}_{ v_j }.
\end{align}
We can write $\widehat{\bm\psi}^\top\bm\eta = \eta_0+\sum_{j=1}^d v_j$, where the $v_j$'s are centered, univariate and mutually orthogonal under the product measure $\lambda_d$. Thus
\begin{equation}
    \|\widehat{\bm\psi}^{\top}\bm\eta\|_{L^2(\lambda_d)}^2 = \eta_0^2+\sum_{j=1}^d\|v_j\|_{L^2(\lambda)}^2.
\end{equation}
The equation \eqref{eq:block-geometry} yields
\begin{equation}
  \eta_0^2 + \frac{1}{\kappa_{\max}^2} \sum\limits_{j=1}^d \sum\limits_{m=1}^M \left( \eta_m^{(j)} \right)^2 \leq \|\widehat{\bm\psi}^{\top}\bm\eta\|_{L^2(\lambda_d)}^2 \leq \eta_0^2 + \frac{1}{\kappa_{\min}^2} \sum\limits_{j=1}^d \sum\limits_{m=1}^M \left( \eta_m^{(j)} \right)^2.
\end{equation}
By using that $\kappa_{\min} < 1 < \kappa_{\max}$, we obtain
\begin{equation}
    \frac1{\kappa_{\max}^2}\|\bm\eta\|_2^2
 \leq\|\widehat{\bm\psi}^{\top}\bm\eta\|_{L^2(\lambda_d)}^2
 \leq\frac1{\kappa_{\min}^2}\|\bm\eta\|_2^2.
\end{equation}
Finally, since $\kappa_{\min} \leq p \leq \kappa_{\max}$, we have
\begin{equation}\label{eq:riesz-bounds}
 C_{\min}\|\bm\eta\|_2^2 \leq\|\widehat{\bm\psi}^{\top}\bm\eta\|^2 =\bm\eta^\top\bm\Gamma\bm\eta \leq C_{\max}\|\bm\eta\|_2^2,
\end{equation}
which proves \eqref{eq:population-gram-bounds}.
For \eqref{eq:row-norm}, use
$|\widehat\psi_m^{(j)}|\leq2\sqrt2/\kappa_{\min}$, sum the squared
coordinates, and recall that $\kappa_{\min} < 1$.
\end{proof}

\subsection{Proof of Lemma \ref{lem:approximation}}\label{subsec:approximation}

Fix $(p,f)$ and $\mathcal D_n^0$, and define $w_j\coloneqq f_j-h_j/\widehat p_j$. Since $\mu_f=\sum_{j=1}^d\int_0^1 f_j\,d\lambda$,
\begin{equation}\label{eq:centered-residual}
 r(\mathbf x)
 =\sum_{j=1}^d\left(w_j(x_j)-\int_0^1w_j(x)\,dx\right).
\end{equation}
The summands in \eqref{eq:centered-residual} are centered and
orthogonal under $\lambda_d$, for every realization of $\mathcal D_n^0$.
Consequently,
\begin{align}\label{eq:residual-sum}
 \|r\|_{L^2(P)}^2
 &\leq \kappa_{\max}\|r\|_{L^2(\lambda_d)}^2\notag\\
 &\leq\kappa_{\max}\sum_{j=1}^d \left\|w_j-\int_0^1w_j\,d\lambda\right\|_{L^2(\lambda)}^2 \\
 &\leq \kappa_{\max}\sum_{j=1}^d\|w_j\|_{L^2(\lambda)}^2.
\end{align}
Using $f_j=g_j/p_j$, we decompose each summand as
\begin{equation}
     w_j=\frac{g_j-h_j}{p_j}
         +h_j\left(\frac1{p_j}-\frac1{\widehat p_j}\right).
\end{equation}
The density lower bounds and $(u+v)^2\leq2u^2+2v^2$ yield
\begin{equation}\label{eq:univariate-approximation}
 \|w_j\|_{L^2(\lambda)}^2 \leq\frac2{\kappa_{\min}^2}\|g_j-h_j\|_{L^2(\lambda)}^2 + \frac2{\kappa_{\min}^4} \int_0^1 h_j(x)^2|\widehat p_j(x)-p_j(x)|^2\,dt.
\end{equation}
By Parseval's identity and \eqref{eq:sobolev-class},
\begin{equation}\label{eq:fourier-tail}
 \|g_j-h_j\|_{L^2(\lambda)}^2 = \sum_{m>M}(\theta_m^{(j)})^2 \leq R^2M^{-2\beta},
 \qquad
 \|h_j\|_{L^2(\lambda)}^2\leq R^2.
\end{equation}
For the fixed pair $(p,f)$, $h_j$ is deterministic. Tonelli's theorem
and \eqref{eq:density-risk-assumption} therefore give
\begin{align}\label{eq:weighted-density-error}
 \mathbb E\!\left[
 \int_0^1h_j(x)^2|\widehat p_j(x)-p_j(x)|^2\,dt\right] &=\int_0^1h_j(x)^2
            \mathbb E[|\widehat p_j(x)-p_j(x)|^2]\,dx\notag\\
 &\leq\rho_{n_0}\|h_j\|_{L^2(\lambda)}^2\notag\\
 &\leq R^2\rho_{n_0}.
\end{align}
Substituting \eqref{eq:fourier-tail} and
\eqref{eq:weighted-density-error} into
\eqref{eq:univariate-approximation}, and summing in
\eqref{eq:residual-sum}, proves the claim.

\begin{remark}[The dependence on dimension and the density norm]
The centering identity \eqref{eq:centered-residual} is responsible for
the linear dependence on $d$: it avoids bounding the squared norm of
a sum by $d$ times the sum of squared norms. The calculation uses
$\sup_t\mathbb E|\widehat p_j(t)-p_j(t)|^2$, not
$\mathbb E\|\widehat p_j-p_j\|_{L^\infty}^2$.
Indeed, the deterministic weights $h_j^2$ allow the pointwise risk bound
to be integrated directly. This uses only $\|h_j\|_{L^2(\lambda)}\leq R$,
so it applies to every $\beta>0$ and allows dependence between the
different marginal estimators.
\end{remark}

\subsection{Bounding the conditional squared bias}\label{subsec:bias}

In this subsection, we prove the inequality of equation \eqref{eq:bias_upper_bound}.

On $\zeta_n$, define
\begin{equation}\label{eq:A-and-B}
 \bm A\coloneqq\widehat{\bm\Psi}^\top\widehat{\bm\Psi}
             =n_1\widehat{\bm\Gamma},
 \qquad
 \bm B\coloneqq\bm A^{-1}\widehat{\bm\Psi}^\top.
\end{equation}
We extend $\bm B$ by zero on $\zeta_n^c$ when writing expressions
with an indicator. On $\zeta_n$,
\begin{equation}
     \bm B\bm B^\top
 =\bm A^{-1}\widehat{\bm\Psi}^\top\widehat{\bm\Psi}\bm A^{-1}
 =\bm A^{-1},
\end{equation}
and therefore
\begin{equation}\label{eq:pseudoinverse-bound}
 \|\bm B\|_{\mathrm{op}}^2
 =\|\bm A^{-1}\|_{\mathrm{op}}
 \leq\frac1{n_1c_1}.
\end{equation}
The triangle inequality and the population norm bound \eqref{eq:riesz-bounds} implies, on this event,
\begin{align}\label{eq:bias-design-bound}
 \|r-\widehat{\bm\psi}^{\top}\bm B\bm r\|^2
 &\leq2\|r\|^2
        +2\|\widehat{\bm\psi}^{\top}\bm B\bm r\|^2\notag\\
 &\leq2\|r\|^2+2C_{\max}\|\bm B\bm r\|_2^2\notag\\
 &\leq2\|r\|^2
       +\frac{2C_{\max}}{n_1 c_1}
           \sum_{i=1}^{n_1}r(\mathbf X_i^1)^2.
\end{align}
Conditional on $\mathcal D_n^0$, $r$ is fixed and the second-sample
covariates are i.i.d. with distribution $P$. Hence
\begin{equation}\label{eq:empirical-residual-expectation}
 \mathbb E\!\left[
 \sum_{i=1}^{n_1}r(\mathbf X_i^1)^2\; \middle| \;\mathcal D_n^0\right]
 =n_1\|r\|^2.
\end{equation}
Multiplying \eqref{eq:bias-design-bound} by $\mathbf 1_{\zeta_n}$,
discarding the indicator on its nonnegative right-hand side, and using
\eqref{eq:empirical-residual-expectation} gives
\begin{equation}\label{eq:accepted-bias}
 \mathbb E\!\left[
 \|r-\widehat{\bm\psi}^{\top}\bm B\bm r\|^2
 \mathbf 1_{\zeta_n}\; \middle| \;\mathcal D_n^0\right]
 \leq\left(2+\frac{2C_{\max}}{c_1}\right)\|r\|^2.
\end{equation}
In particular, no concentration inequality for the approximation
residual is needed. Also, the two terms inside the norm on the left
are not asserted to be orthogonal in $L^2(P)$.

\subsection{Bounding the conditional variance}\label{subsec:variance}

In this subsection, we prove the inequality of equation \eqref{eq:var_upper_bound}.

The prediction loss additionally involves the population Gram matrix.
Recall also that $\bm B\bm B^\top=\bm (\widehat{\bm\Psi}^\top\widehat{\bm\Psi})^{-1}$ on $\zeta_n$,
\begin{align}\label{eq:prediction-variance}
 \mathbb E\!\left[ \|\widehat{\bm\psi}^{\top}\bm B\bm\xi\|^2 \; \middle| \;\mathcal D_n\right] &= \mathbb E\!\left[ (\bm B\bm\xi)^\top\bm\Gamma(\bm B\bm\xi) \; \middle| \;\mathcal D_n\right]\notag\\
 &= \mathbb{E}\left[ \bm \xi^\top \bm B^\top \bm \Gamma \bm B \bm \xi \; \middle| \;\mathcal D_n \right] \notag\\
 &= \operatorname{Tr}(\mathbb{E}\left[ \bm \xi^\top \bm B^\top \bm \Gamma \bm B \bm \xi \; \middle| \;\mathcal D_n \right]) \notag\\
 &= \operatorname{Tr}( \bm B^\top \bm \Gamma \bm B \mathbb{E}\left[ \bm \xi \bm \xi^\top \; \middle| \;\mathcal D_n \right]) \notag\\
 &=\sigma^2\operatorname{Tr}(\bm\Gamma\bm B\bm B^\top)\notag\\
 &=\sigma^2 \operatorname{Tr}(\bm\Gamma ( \widehat{\bm\Psi}^\top\widehat{\bm\Psi})^{-1}).
\end{align}
Using that $\bm\Gamma\preceq C_{\max}\bm I_K$ and recalling that on $\zeta_n$ we have $\widehat{\bm\Psi}^\top\widehat{\bm\Psi}\succeq n_1c_1\bm I_K $, this leads to the desired result. And because $\zeta_n$ is measurable with respect to $\mathcal D_n$, the tower property gives
\begin{equation}\label{eq:accepted-variance}
 \mathbb E\!\left[
 \|\widehat{\bm\psi}^{\top}\bm B\bm\xi\|^2
 \mathbf 1_{\zeta_n}\; \middle| \;\mathcal D_n^0\right]
 \leq\frac{\sigma^2C_{\max}K}{n_1c_1}.
\end{equation}

\subsection{Risk on the rejected event}\label{subsec:rejected-risk}

We use only the following lower-tail form of the matrix Chernoff
inequality. We state the precise result and its parameters to make the
dependence on $K$ explicit.

\begin{lemma}[Rejection probability]\label{lem:rejection}
Let define the following quantity
\begin{equation}
    \Delta_{n,K} \coloneqq \min\left\{1,K\exp\!\left(-c_2\frac{n_1}{K}\right)\right\}.
\end{equation}
Almost surely with respect to the first sample, we have the following upper bound
\begin{equation}\label{eq:rejection-probability}
 \mathbb P_{p,f}(\zeta_n^c\mid\mathcal D_n^0) \leq \Delta_{n,K}.
\end{equation}
\end{lemma}
\begin{proof}
Condition on $\mathcal D_n^0$ and set
\begin{equation}
    \bm S_i\coloneqq
 \widehat{\bm\psi}(\mathbf X_i^1)
 \widehat{\bm\psi}(\mathbf X_i^1)^\top,
 \qquad 1\leq i\leq n_1.
\end{equation}
These matrices are conditionally independent and positive semidefinite.
Their sum is $n_1\widehat{\bm\Gamma}$.
The only nonzero eigenvalue of a matrix $\bm v\bm v^\top$ is
$\|\bm v\|_2^2$. Thus Lemma~\ref{lem:geometry} verifies the hypotheses
of Theorem~\ref{thm:tropp}, conditionally on $\mathcal D_n^0$, with
\begin{equation}\label{eq:tropp-parameters}
 N=n_1,\qquad
 b_0=\frac{8K}{\kappa_{\min}^2},\qquad
 \mu_{\min}=n_1\lambda_{\min}(\bm\Gamma)
                  \geq n_1C_{\min}.
\end{equation}
Since $c=C_{\min}/2$, we have the inclusion
\begin{equation}
    \zeta_n^c
 \subseteq\left\{
 \lambda_{\min}\!\left(\sum_{i=1}^{n_1}\bm S_i\right)
                         \leq\mu_{\min}/2\right\}.
\end{equation}
Applying \eqref{eq:tropp-half} and using
\eqref{eq:tropp-parameters},
\begin{align}
 \mathbb P(\zeta_n^c\mid\mathcal D_n^0)
 &\leq K\exp\!\left(
 -\frac{1-\log2}{2}\frac{n_1C_{\min}\kappa_{\min}^2}{8K}
 \right)\\
 &=K\exp\!\left(-c_2\frac{n_1}{K}\right).
\end{align}
Taking the minimum with $1$ proves the claim.
\end{proof}

The Riesz bounds and the Sobolev smoothness yield the following inequality for every $(p,f)\in\mathcal C_{\beta,\gamma}$,
\begin{equation}\label{eq:signal-bound}
 \|f\|^2\leq C_{\max}dR^2.
\end{equation}

Because $\widehat f_M=0$ on $\zeta_n^c$, Lemmas~\ref{lem:rejection}
and this last inequality give, conditionally on $\mathcal D_n^0$,
\begin{align}\label{eq:rejected-risk}
 \mathbb E\!\left[ \|f-\widehat f_M\|^2\mathbf 1_{\zeta_n^c}\mid\mathcal D_n^0\right]
 &=\|f\|^2 \mathbb P(\zeta_n^c\mid\mathcal D_n^0) \\
 &\leq C_{\max}dR^2\Delta_{n,K}.
\end{align}
In particular, the rejection contribution involves the squared
$L^2(P)$ norm of the signal, rather than a supremum norm.

\subsection{Combining the bounds}\label{subsec:assembly}

In this subsection we combine all the bound and complete the proof of Theorem \ref{thm:main-unknown-density-upper}.

First, by integrating the equation \eqref{eq:conditional-bias-variance} over $\mathcal D_n^1$, we obtain the following identity
\begin{equation}
\mathbb E\!\left[ \|f-\widehat f_M\|^2\mathbf 1_{\zeta_n} \; \middle| \;\mathcal D_n^0\right] = \mathbb E\!\left[
 \|r-\widehat{\bm\psi}^{\top}\bm B\bm r\|^2
 \mathbf 1_{\zeta_n}\; \middle| \;\mathcal D_n^0\right] + \mathbb E\!\left[
 \|\widehat{\bm\psi}^{\top}\bm B\bm\xi\|^2
 \mathbf 1_{\zeta_n}\; \middle| \;\mathcal D_n^0\right].
\end{equation}

Equations~\eqref{eq:accepted-bias} and~\eqref{eq:accepted-variance} imply
\begin{equation}
\mathbb E\!\left[ \|f-\widehat f_M\|^2\mathbf 1_{\zeta_n} \; \middle| \;\mathcal D_n^0\right] \leq \left(2+\frac{2C_{\max}}{c_1}\right)\|r\|^2 +\frac{\sigma^2C_{\max}K}{n_1c_1}.
\end{equation}
Adding \eqref{eq:rejected-risk} and integrating over $\mathcal D_n^0$ yields
\begin{equation}
 \mathbb E\!\left[\|f-\widehat f_M\|^2\right] \leq\left(2+\frac{4C_{\max}}{C_{\min}}\right) \mathbb E_{p,f}[\|r\|^2] + \frac{2\sigma^2C_{\max}K}{n_1 C_{\min}} + C_{\max}dR^2\Delta_{n,K}.
\end{equation}

Substitute the approximation bound of Lemma~\ref{lem:approximation} yields this final upper bound

\begin{equation}\label{eq:risk-before-approximation}
\boxed{
 \mathbb E\!\left[\|f-\widehat f_M\|^2\right] \leq\left(2+\frac{4C_{\max}}{C_{\min}}\right) 2\kappa_{\max}dR^2
 \left(\frac{M^{-2\beta}}{\kappa_{\min}^2}
           +\frac{\rho_{n_0}}{\kappa_{\min}^4}\right) + \frac{2\sigma^2C_{\max}K}{n_1 C_{\min} } + C_{\max}dR^2\Delta_{n,K}}
\end{equation}

Every term on the resulting right-hand side is independent of the
particular pair $(p,f)$, which justifies taking the supremum over the
whole class and allows us to obtain a minimax upper bound. Combining this inequality with $n_0,n_1\asymp n$ and $\rho_{n_0} \lesssim n^{ - \frac{2 \gamma}{2 \gamma + 1} }$ yields an upper bound of the following form
\begin{equation}\label{eq:rate-with-remainder}
 \sup_{(p,f)\in\mathcal C_{\beta,\gamma}}
 \mathbb E_{p,f}[\|\widehat f_M-f\|^2]
 \lesssim
 d \, M^{-2\beta}
 + \frac{dM}{n}
 +d \, n^{ - \frac{2 \gamma}{2 \gamma + 1} }
 +d \, \Delta_{n,K}.
\end{equation}
This bound is valid without a growth condition on $d$; the rejection
term is retained explicitly.

Balancing the Sobolev approximation term and the variance gives
$M\asymp n^{\frac{1}{2\beta + 1}}$. Under Assumption~\ref{assu:dimension},
\begin{equation}
     d\log n
    =
    o\!\left(n^{\frac{2\beta}{2\beta+1}}\right),
\end{equation}
and therefore
\begin{equation}
    K\log n  \asymp dM\log n =o(n).
\end{equation}
Since $n_1\asymp n$ and $c_2>0$ is fixed, it follows that
\begin{equation}
    \frac{c_2n_1}{K\log n}\longrightarrow\infty.
\end{equation}
Moreover, $K\leq n$ for all sufficiently large $n$. Using $\Delta_{n,K}\leq K\exp(-c_2n_1/K)$, we obtain
\begin{align}
    n^{\frac{2\beta}{2\beta+1}}\Delta_{n,K}
    &\leq
    \exp\left(
        \frac{2\beta}{2\beta+1}\log n
        +\log K-\frac{c_2n_1}{K}
    \right)\\
    &\leq
    \exp\left[
        \log n\left(
            1+\frac{2\beta}{2\beta+1}
            -\frac{c_2n_1}{K\log n}
        \right)
    \right]
    \longrightarrow0.
\end{align}
Consequently,
\begin{equation}
    d\,\Delta_{n,K}
    =
    o\!\left(
        d\,n^{-\frac{2\beta}{2\beta+1}}
    \right).
\end{equation}
We finally obtain the desired minimax upperbound
\begin{equation}
\boxed{
 \inf_{\widetilde f}\sup_{(p,f)\in\mathcal C_{\beta,\gamma}} \mathbb E_{p,f}[\|\widetilde f-f\|^2] \lesssim d\,n^{ - \frac{2 \beta}{ 2 \beta + 1 } }
         +d\,n^{- \frac{2 \gamma}{ 2 \gamma + 1 }}.}
\end{equation}

\subsection{Construction of the marginal-density estimators}
\label{subsec:density-construction}

We now establish the density-estimation guarantee used in the
regression analysis. We use piecewise polynomial projection followed
by pointwise clipping. This construction handles the boundary of
$[0,1]$ without extending the marginal densities beyond their domain.

Throughout this subsection, write
\begin{equation}
     s = \lfloor\gamma\rfloor, \qquad \alpha=\gamma-s\in[0,1).
\end{equation}
Under our definition of $(\gamma,L)$-Hölder regularity, each marginal
density $p_j$ belongs to $C^s([0,1])$ and satisfies
\begin{equation}
     |p_j^{(s)}(u)-p_j^{(s)}(v)| \le L|u-v|^\alpha, \qquad u,v\in[0,1],\quad u\ne v.
\end{equation}
Here $p_j^{(0)}=p_j$. When $\gamma$ is an integer, $\alpha=0$
and this condition bounds the oscillation of $p_j^{(s)}$ by $L$.
The marginal density bounds
$\kappa_{\min}\le p_j\le\kappa_{\max}$ follow from the
joint density bounds defining $\mathcal C_{\beta,\gamma}$.

\begin{lemma}[Uniform pointwise density risk]
\label{lem:density-estimation}
Let $s=\lfloor\gamma\rfloor$ and $n_0\ge1$.
For every integer $J\ge1$, there exist estimators
$\widehat p_j$, depending on $J$ and constructed from
$\mathcal D_n^0$, satisfying
\eqref{eq:estimated-density-bounds} and
\begin{equation}
\label{eq:density-explicit-risk}
    \max_{1\le j\le d}\sup_{t\in[0,1]}\mathbb E_{p}\!\left[|\widehat p_j(t)-p_j(t)|^2 \right] \le \left(\frac{(s+2)L}{s!}\right)^2J^{-2\gamma} +\kappa_{\max}(s+1)^2\frac{J}{n_0}.
\end{equation}
In particular, choosing
$J=\lceil n_0^{ \frac{1}{2\gamma + 1} }\rceil$
provides the guarantee
\eqref{eq:density-risk-assumption} with a deterministic bound
$\rho_{n_0}$ satisfying
\begin{equation}
\label{eq:density-rate}
    \rho_{n_0} \le C_{\gamma,L,\kappa_{\max}}\, n_0^{- \frac{2\gamma}{2 \gamma + 1} } \lesssim n_0^{ -\frac{2 \gamma}{2\gamma + 1} }
\end{equation}
where the underlying constant is independent of $n,d$ and the pair $(p,f)$.
\end{lemma}

\begin{proof}
Fix $(p,f)\in\mathcal C_{\beta,\gamma}$.
The estimators below depend only on the first-subsample covariates,
so their distributions depend on $p$ but not on $f$ and the expectations will be taken over $p$.

Let be a integer $J \geq 1$, partition $[0,1]$ into intervals of length $1/J$ by setting for every $k \in \{1, \dots, J-1\}$
\begin{equation}
    I_k \coloneqq \left[\frac{k-1}{J},\frac{k}{J}\right[,
\end{equation}
and 
\begin{equation}
    I_J \coloneqq \left[\frac{J-1}{J},1\right].
\end{equation}
Thus every point belongs to exactly one interval, including the endpoints.

Let $\mathcal L_0,\ldots,\mathcal L_s$ be the shifted Legendre polynomials normalized to form an orthonormal basis of the polynomials of degree at most $s$ in $L^2(\lambda)$ (we are working on $[0,1]$). They satisfy
\begin{equation}
\label{eq:legendre-bound}
    |\mathcal L_r(u)|\le\sqrt{2r+1}, \qquad  \sum_{r=0}^s\mathcal L_r(u)^2\le(s+1)^2,
    \qquad u\in[0,1],
\end{equation}
see for example \citet{brezis2011functional} or \citet{szeg1939orthogonal}.
Define the piecewise polynomial basis functions
\begin{equation}
    \varphi_{kr}(t) \coloneqq \sqrt{J}\,\mathcal L_r(Jt-k+1)\mathbf 1_{I_k}(t), \qquad 1\le k\le J,\quad 0\le r\le s.
\end{equation}
These functions form an orthonormal basis in $L^2(\lambda)$ of the space of functions that are polynomial of degree at most $s$ on each interval. The associated projection kernel is
\begin{equation}
     \bm K_J(t,u) \coloneqq \sum_{k=1}^J\sum_{r=0}^s \varphi_{kr}(t)\varphi_{kr}(u).
\end{equation}

For each coordinate $j$, define
\begin{equation}
\label{eq:density-estimators}
    \begin{aligned}
        \widetilde p_j(t)
        &\coloneqq
        \frac{1}{n_0}
        \sum_{i=1}^{n_0}\bm K_J(t,X_{ij}^0),\\
        \widehat p_j(t)
        &\coloneqq \operatorname{clip}\left( \kappa_{\min} , \widetilde p_j(t) , \kappa_{\max} \right),
    \end{aligned}
\end{equation}
which imposes that $\kappa_{\min} \leq \widehat p_j \leq \kappa_{\max}$ (which are known parameters of the class $\mathcal C_{\beta , \gamma}$).
\begin{remark}
    The operator $\operatorname{clip}$ is defined as follows
    \begin{equation}
    \operatorname{clip}(a , x , b) \coloneqq
        \begin{cases}
            x \qquad \text{if } x \in [a,b], \\
            a \qquad \text{if } x < a, \\
            b \qquad \text{if } x > b.
        \end{cases}
    \end{equation}
\end{remark}

Here $X_{ij}^0$ is the $j$th coordinate of $\mathbf X_i^0$. The clipped estimator is pointwise contractive, \textit{i.e.}
\begin{equation}
\label{eq:clipping-contraction}
    |\widehat p_j(t)-p_j(t)| \le |\widetilde p_j(t)-p_j(t)|.
\end{equation}
Indeed, there are three cases. First, $\widetilde p_j(t) = \widehat p_j(t)$ which implies the equality. Second, we have $\widehat p_j(t) = \kappa_{\min}$ which implies by clipping that $ \widehat p_j(t) \geq \widetilde p_j(t) $ which leads to
\begin{equation}
    \widetilde p_j(t) - p_j(t) \leq \widehat p_j(t) - p_j(t) = \kappa_{\min} - p_j(t) \leq 0.
\end{equation}
By applying the absolute value we obtain the desired inequality. Observe that the third case is symmetric. It therefore suffices to bound the squared bias and variance
of $\widetilde p_j(t)$.

\paragraph{Bias.}
Let $\Pi_J$ denote the orthogonal projection associated with $\bm K_J$, with pointwise representative
\begin{equation}
    (\Pi_Jv)(t) = \int_0^1\bm K_J(t,u)v(u)\,du =\sum\limits_{k}\sum\limits_{r} \langle \varphi_{kr} , v \rangle_{ L^2(\lambda) } \, \varphi_{kr}(t).
\end{equation}
Since each $X_{ij}^0$ has density $p_j$, we can write
\begin{equation}
    \mathbb E_p[\widetilde p_j(t)] = \int_0^1 \bm K_J(t,u)p_j(u)\,du = (\Pi_Jp_j)(t).
\end{equation}
For $t\in I_k$, the function $\bm K_J(t,\cdot)$ vanishes outside $I_k$. By Cauchy--Schwarz, orthonormality, and equation \eqref{eq:legendre-bound}, we have
\begin{align}
    \int_{I_k}|\bm K_J(t,u)|\,du
    &\le J^{-1/2} \left( \int_{I_k}\bm K_J(t,u)^2\,du \right)^{1/2}\\
    &\leq  J^{-1/2} \left( \sum_{r=0}^s\varphi_{kr}(t)^2 \right)^{1/2}\\
    &\le s+1.
\end{align}

Let $a_k \coloneqq  (k-1)/J$, and let
\begin{equation}
    T_{jk}(t) \coloneqq \sum_{r=0}^s \frac{p_j^{(r)}(a_k)}{r!}(t-a_k)^r
\end{equation}
be the Taylor degree $s$ polynomial of $p_j$ at point $a_k$. We will show that
\begin{equation}
\label{eq:taylor-bound}
    \sup_{t\in I_k}|p_j(t)-T_{jk}(t)|
    \le
    \frac{L}{s!}J^{-\gamma}.
\end{equation}
If $s=0$, then $0<\gamma<1$ and $T_{jk}=p_j(a_k)$.
The Hölder condition directly gives
\begin{equation}
    |p_j(t)-T_{jk}(t)| \le L|t-a_k|^\gamma \le LJ^{-\gamma}.
\end{equation}
If $s\ge1$, Taylor's integral formula gives
\begin{equation}
    p_j(t) - T_{jk}(t) = \frac{1}{(s-1)!} \int_{a_k}^t (t-u)^{s-1} \bigl(p_j^{(s)}(u)-p_j^{(s)}(a_k)\bigr) \,du.
\end{equation}
For $u\in[a_k,t]$, the derivative difference is bounded by $LJ^{-\alpha}$. When $\alpha=0$, this follows from the oscillation bound; when $\alpha>0$, it follows from $|u-a_k|\le1/J$ and the Hölder condition.
Consequently,
\begin{align}
    |p_j(t)-T_{jk}(t)|
    &\le
    \frac{LJ^{-\alpha}}{(s-1)!}
    \int_{a_k}^t(t-u)^{s-1}\,du\\
    &=
    \frac{LJ^{-\alpha}}{s!}(t-a_k)^s\\
    &\le
    \frac{L}{s!}J^{-(s+\alpha)} \\
    &= \frac{L}{s!}J^{-\gamma}.
\end{align}
This proves \eqref{eq:taylor-bound} for every $\gamma>0$,
including integer values.

The projection reproduces polynomials of degree at most $s$
on each cell. Therefore, for $t\in I_k$, we have
\begin{equation}
    \Pi_J T_{jk}(t) = T_{jk}(t),
\end{equation}
which yields
\begin{align*}
    |(\Pi_Jp_j)(t) - p_j(t)|
    &= \left| (\Pi_Jp_j)(t) - (\Pi_J T_{jk})(t) + T_{jk}(t) - p_j(t) \right| \\
    &=\left|
        \int_{I_k}\bm K_J(t,u)
            \bigl(p_j(u)-T_{jk}(u)\bigr)\,du
        -\bigl(p_j(t)-T_{jk}(t)\bigr)
    \right|\\
    &\le
    \left(
        1+\int_{I_k}|\bm K_J(t,u)|\,du
    \right)
    \sup_{u\in I_k}|p_j(u)-T_{jk}(u)|\\
    &\le
    \frac{(s+2)L}{s!}J^{-\gamma}.
\end{align*}
Taking the supremum over all cells gives
\begin{equation}
\label{eq:density-bias-bound}
    \sup_{t\in[0,1]}
    \left|
        \mathbb E_p[\widetilde p_j(t)]-p_j(t)
    \right|
    \le
    \frac{(s+2)L}{s!}J^{-\gamma}.
\end{equation}

\paragraph{Variance.}
For each fixed coordinate $j$, the observations $X_{1j}^0,\ldots,X_{n_0j}^0$ are independent and have density $p_j$. We have
\begin{align}
\label{eq:density-variance-bound}
    \operatorname{Var}_p(\widetilde p_j(t)) &= \frac{1}{n_0} \operatorname{Var}_p(\bm K_J(t,X_{1j}^0)) \\
    &\le \frac{1}{n_0} \int_0^1\bm K_J(t,u)^2p_j(u)\,du \\
    &\le \frac{\kappa_{\max}}{n_0} \int_0^1 \bm K_J(t,u)^2\,du \\
    &\leq \frac{\kappa_{\max}}{n_0} \sum_{k=1}^J\sum_{r=0}^s\varphi_{kr}(t)^2 \\
    & \leq\kappa_{\max}(s+1)^2\frac{J}{n_0}.
\end{align}
The last inequality uses the fact that exactly one cell contains $t$, together with \eqref{eq:legendre-bound}.

\paragraph{Risk bound.}
Combining the clipping inequality with the standard bias--variance decomposition yields
\begin{align}
    \mathbb E_{p,f}\!\left[
        |\widehat p_j(t)-p_j(t)|^2
    \right]
    &\le
    \mathbb E_p\!\left[
        |\widetilde p_j(t)-p_j(t)|^2
    \right]\\
    &=
    \left(
        \mathbb E_p[\widetilde p_j(t)]-p_j(t)
    \right)^2
    +
    \operatorname{Var}_p(\widetilde p_j(t))\\
    &\le
    \left(\frac{(s+2)L}{s!}\right)^2J^{-2\gamma}
    +
    \kappa_{\max}(s+1)^2\frac{J}{n_0}.
\end{align}
The right-hand side is independent of $t$, $j$, and
the pair $(p,f)$. Taking the corresponding suprema
proves \eqref{eq:density-explicit-risk}.

Finally, for $J=\lceil n_0^{ \frac{1}{2\gamma + 1} }\rceil$ and
$n_0\ge1$,
\[
    J^{-2\gamma} \le n_0^{-\frac{2\gamma}{2\gamma + 1} }, \qquad
    \frac{J}{n_0} \le 2n_0^{- \frac{2\gamma}{2\gamma + 1} }.
\]
Thus \eqref{eq:density-rate} holds with
\[
    C_{\gamma,L,\kappa_{\max}}
    =
    \left(\frac{(s+2)L}{s!}\right)^2
    +
    2\kappa_{\max}(s+1)^2.
\]
Since $s=\lfloor\gamma\rfloor$, this constant depends
only on $\gamma$, $L$, and $\kappa_{\max}$.
\end{proof}

\subsection{Proof of Theorem \ref{thm:component_recovery}}

Let first $\widehat{\bm\theta}_M$ collect the fitted coefficients $\widehat\theta_m^{(j)}$ in dictionary order, and let $\bm\theta_M$ collect the corresponding true coefficients. With these notations, recall that we have
\begin{equation}
    \widehat{\bm\eta}_M
    =
    \begin{pmatrix}
        \widehat \eta_0 \\
        \widehat{\bm\theta}_M
    \end{pmatrix},
    \qquad
    \bm\eta_M^\star
    =
    \begin{pmatrix}
        \mu_f\\
        \bm\theta_M
    \end{pmatrix},
    \qquad
    \mu_f=\int_{[0,1]^d}f(x)\,dx,
\end{equation}
where $\widehat \eta_0$ is the $0^{\text{th}}$ component of the $K-$dimensional vector $\widehat{\bm\eta}_M$ defined in \eqref{eq:main-thresholded-estimator}. By construction, recall that we have
\begin{equation}
    \widehat f_M = \widehat{\bm\psi}^{\top}\widehat{\bm\eta}_M, 
    \qquad
    t_M = \widehat{\bm\psi}^{\top}\bm\eta_M^\star,
    \qquad
    r = f - t_M.
\end{equation}
For every realization of the full sample, using the Riesz bounds of equation \eqref{eq:riesz-bounds}, we always have
\begin{align}\label{eq:component-coefficient-control}
    \|\widehat{\bm\theta}_M-\bm\theta_M\|_2^2
    &\leq
    \|\widehat{\bm\eta}_M-\bm\eta_M^\star\|_2^2
    \notag\\
    &\leq
    \frac{1}{C_{\min}}
    \|\widehat f_M-t_M\|^2
    \notag\\
    &\leq \frac{2}{C_{\min}}\left(\|\widehat f_M-f\|^2 + \|r\|^2 \right).
\end{align}

For each component $j$, recall that the rectified estimator of $f_j$ is defined as
\begin{equation}
    \bar f_{j,M} \coloneqq \frac{1}{\widehat p_j} \sum\limits_{m = 1}^M \widehat \theta_m^{(j)} \phi_m. 
\end{equation}
and recall the notations
\begin{equation}
    g_j=p_jf_j,
    \qquad
    h_j=\sum_{m=1}^M\theta_m^{(j)}\phi_m,
    \qquad
    \widehat h_j
    =
    \sum_{m=1}^M\widehat\theta_m^{(j)}\phi_m.
\end{equation}
Set $w_j = f_j-h_j/\widehat p_j$ in order to obtain
\begin{equation}
    f_j - \bar f_{j,M} = w_j+\frac{h_j-\widehat h_j}{\widehat p_j}.
\end{equation}
Since $ \kappa_{\min} \leq p_j , \widehat p_j \leq \kappa_{\max} $, orthonormality of the trigonometric basis under Lebesgue measure gives
\begin{equation}\label{eq:component-error-decomposition}
    \sum_{j=1}^d
    \|f_j - \bar f_{j,M}\|^2 \leq 2\sum_{j=1}^d\|w_j\|^2 + 2 \frac{\kappa_{\max}}{\kappa_{\min}^2} \|\widehat{\bm\theta}_M-\bm\theta_M\|_2^2.
\end{equation}
And the upper bound on $\|\widehat{\bm\theta}_M-\bm\theta_M\|_2^2$ yields finally:
\begin{equation}
    \sum_{j=1}^d
    \mathbb{E}_{p,f}\left[\|f_j-\bar f_{j,M}\|^2\right] \leq 2 \kappa_{\max} \sum_{j=1}^d \mathbb{E}_{p}\left[ \|w_j\|^2_{L^2(\lambda)} \right] + 4 \frac{\kappa_{\max}}{\kappa_{\min}^2 C_{\min} } \left( \mathbb{E}_{p,f}\left[ \|\widehat f_M-f\|^2 \right] + \mathbb{E}_{p}\left[\|r\|^2\right] \right).
\end{equation}
Observe that this last expression involves known bounds. Indeed, we showed in subsection \ref{subsec:approximation} that
\begin{equation}
    \mathbb{E}_{p}\left[\| w_j \|_{ L^2(\lambda) }^2\right] \leq 2 R^2 \left( \frac{M^{-2\beta}}{ \kappa_{\min}^2 } + \frac{\rho_{n_0}}{ \kappa_{\min}^4 } \right)  , \qquad \mathbb{E}_{p}\left[ \| r \|^2 \right] \leq 2 \kappa_{\max} d R^2 \left( \frac{M^{-2\beta}}{ \kappa_{\min}^2 } + \frac{\rho_{n_0}}{ \kappa_{\min}^4 } \right)
\end{equation}

By combining the bounds and using that
\begin{equation}
    M \asymp n^{ \frac{1}{2\beta + 1} }, \qquad \rho_{n_0} \lesssim n^{ -\frac{2\gamma}{2\gamma + 1} },
\end{equation}
we obtain the desired rate in $d \, n^{ - \frac{2 \beta}{ 2 \beta + 1 } } + d \, n^{ - \frac{2 \gamma}{ 2 \gamma + 1 } }$.

\section{Proof of Theorem \ref{thm:unknown-density-lower-smooth}}
\label{appendix:lower_smooth}

Let $p_{\text{unif}}\equiv1$ denote the uniform density on $[0,1]^d$. Since $p_{\text{unif}}\in\mathcal P_{\gamma,L}$, the definition of the coupled class gives
\begin{equation}
    \{p_{\text{unif}}\}\times\mathcal F_{\beta,R}(p_{\text{unif}}) \subseteq \mathcal C_{\beta,\gamma}.
\end{equation}
For every estimator $\widehat f$, restricting the supremum to this subclass yields
\begin{equation}
    \sup_{(p,f)\in\mathcal C_{\beta,\gamma}} \mathbb E_{p,f}\!\left[ \|f-\widehat f\|^2 \right] \geq \sup_{f\in\mathcal F_{\beta,R}(p_{\text{unif}})} \mathbb E_{p_{\text{unif}},f}\!\left[ \|f-\widehat f\|_{L^2(\lambda_d)}^2 \right].
\end{equation}
Taking the infimum over estimators, and allowing the estimator in the restricted problem to know $p_{\text{unif}}$, gives
\begin{equation}\label{eq:uniform-design-restriction}
    \mathfrak M(n,d,\mathcal C_{\beta,\gamma})
    \geq
    \mathfrak M(n,d,\mathcal F_{\beta,R}(p_{\text{unif}})).
\end{equation}

Under $p_{\text{unif}}$, all marginal densities equal one, and hence
$\psi_m^{(j)}=\phi_m$.
The class $\mathcal F_{\beta,R}(p_{\text{unif}})$ is therefore the
classical additive Sobolev class \citep{tsybakov2009introduction}, with radius $R$ for each component. Moreover, both Riesz constants equal one. The fixed-density lower bound established in Section~\ref{sec:fixed_density} consequently implies
\begin{equation}
    \mathfrak M(n,d,\mathcal F_{\beta,R}(p_{\text{unif}})) \gtrsim d\,n^{ -\frac{2 \beta}{ 2 \beta + 1 } }.
\end{equation}
where the underlying constant depends only on $\beta,R,\sigma$. Combining this inequality with \eqref{eq:uniform-design-restriction} proves the Theorem \ref{thm:unknown-density-lower-smooth}.

\section{Proof of Theorem \ref{thm:unknown-density-lower-rough}}
\label{appendix:lower_rough}

Recall that the minimax expected prediction risk by
\begin{equation}
    \mathfrak M(n,d,\mathcal C_{\beta,\gamma})
    =\inf_{\widehat f}\sup_{(p,f)\in\mathcal C_{\beta,\gamma}}
      \mathbb E_{p,f}
       \bigl[\|\widehat f-f\|^2\bigr].
\end{equation}
The estimator is one measurable function of the observations, used for every pair in the class. It need not be additive. Both the observation law and the norm in the loss vary with $(p,f)$. The proof is long and technical, that is why we present in the following a \emph{sketch of proof}.

We construct a finite family $\{(p_{\boldsymbol\omega},F_{\boldsymbol\omega})\}_{\boldsymbol\omega} \subset\mathcal C_{\beta,\gamma}$ whose regression functions are sufficiently separated, while neighboring observation laws remain difficult to distinguish. The main challenge is to make this separation depend on $\gamma$ while preserving the Sobolev constraint on the density-weighted components. To this end, we fix a smooth, mean-zero function $g$ with a
nonzero plateau and define
\begin{equation}
    F_{\boldsymbol\omega,j} = \frac{g}{p_{\boldsymbol\omega,j}},
    \qquad
    F_{\boldsymbol\omega} = \sum_{j=1}^d F_{\boldsymbol\omega,j}.
\end{equation}
Thus $p_{\boldsymbol\omega,j}F_{\boldsymbol\omega,j}=g$, so the weighted components satisfy the Sobolev constraint and the centering condition uniformly over the family. All variation is introduced through the design densities. These are constructed using a bounded sinusoidal coupling of localized odd bumps of width $h$ and amplitude proportional to $h^\gamma$. This construction preserves the joint density bounds independently of $d$ and produces $\gamma$-Hölder marginals. The dimension assumption ensures that the marginal perturbations retain the required amplitude.

Let $P_{\boldsymbol\omega}$ denote the design distribution and $Q_{\boldsymbol\omega}$ the law of one observed pair. For adjacent vertices, Lemma~\ref{lem:e2e-gaussian-kl} gives
\begin{equation}
    \operatorname{KL}\!\left(
        Q_{\boldsymbol\omega}^{\otimes n} \|
        Q_{\boldsymbol\omega^{(k)}}^{\otimes n}
    \right) =
    n\operatorname{KL}\!\left(
        P_{\boldsymbol\omega} \|
        P_{\boldsymbol\omega^{(k)}}
    \right)
    +\frac{n}{2\sigma^2}
    \|F_{\boldsymbol\omega}
      -F_{\boldsymbol\omega^{(k)}}\|_{L^2(P_{\boldsymbol\omega})}^2.
\end{equation}
The uniform lower density bound and Lemma~\ref{lem:e2e-kl-bounds} control the design divergence by the squared $L^2(\lambda_d)$ distance between the densities. The construction makes this distance, as well as the neighboring regression distance, of order at most $h^{2\gamma+1}$. Consequently, choosing $h\asymp n^{-\frac{1}{2\gamma + 1}}$ and a sufficiently small fixed bump amplitude constant bounds every neighboring sample divergence by $1/8$.

Finally, the odd bumps and the plateau of $g$ yield an exact affine representation of the regression family with $N\asymp d/h$ mutually orthogonal directions in $L^2(\lambda_d)$, each having squared norm of order
$h^{2\gamma+1}$. Since $p_{\boldsymbol\omega}\geq\kappa_{\min}$, the prediction loss uniformly dominates the squared $L^2(\lambda_d)$ loss. Applying Lemma~\ref{lem:e2e-assouad} therefore gives
\begin{equation}
    \mathfrak M(n,d,\mathcal C_{\beta,\gamma})
    \gtrsim
    \frac{d}{h}\,h^{2\gamma+1}
    \asymp
    d\,n^{-\frac{2\gamma}{2\gamma+1}},
\end{equation}
under the stated dimension assumption. This establishes the density-driven lower bound in the rough regime $\gamma<\beta$.

\begin{remark}[An alternative approach via Fano's inequality]
A standard approach to minimax lower bounds is to combine a Varshamov--Gilbert packing with Fano's inequality (see for example \citet{tsybakov2009introduction}). We instead use Assouad's lemma, which directly exploits the hypercube structure of our construction.

The same family could also support a proof based on Fano's inequality. Writing $N$ for the number of sign coordinates, one would first select an exponentially large subset $\Omega\subset\{-1,+1\}^N$ whose distinct elements have Hamming distance of order $N$. Orthogonality would then give a squared $L^2(\lambda_d)$ separation of order $Nv_h^2$ between the corresponding regression functions. The additional step would be to control the sample KL divergence between arbitrary selected alternatives, rather than only between neighboring vertices. Specifically, one could establish
\begin{equation}
    \operatorname{KL}\!\left(
        Q_{\boldsymbol\omega}^{\otimes n} \|
        Q_{\boldsymbol\omega'}^{\otimes n}
    \right)
    \leq
    C a^2 n h^{2\gamma+1}
    \rho(\boldsymbol\omega,\boldsymbol\omega'),
\end{equation}
where $\rho$ is the Hamming distance and $a$ is the bump-amplitude constant. With $nh^{2\gamma+1}\lesssim1$ and $a$ sufficiently small, these divergences would be bounded by a small multiple of $\log|\Omega|\asymp N$, allowing Fano's inequality to yield the same lower rate. Assouad's lemma avoids this packing step and requires only the neighboring-divergence bounds proved below.
\end{remark}

\subsection{Smooth functions construction}

In this subsection, we show how to construct functions with $\gamma-$Hölder smoothness.

\begin{lemma}
\label{lem:e2e-smooth-functions}
There exist a nonzero $b\in C^\infty((0,1))$ and a smooth periodic real function $w$ such that
\begin{equation}
    b(1-t)=-b(t),\qquad \|b\|_\infty\le1,
    \qquad B_2:=\int_0^1b(t)^2\,dt>0,
\end{equation}
and
\begin{equation}
    \int_0^1w(x)\,dx=0,\qquad
    w(x)=1\quad\text{on }[1/4,1/2].
\end{equation}
For every $\beta>0$, $\sum_{m\ge1}m^{2\beta} |\langle w,\phi_m\rangle|^2<\infty$. Consequently, for every $R>0$, there exists $\tau>0$ such that $g=\tau w$ belongs to the Sobolev ellipsoid of radius $R$.
\end{lemma}

\begin{proof}
Choose a nonzero nonnegative function $\eta\in C^\infty((1/8,3/8))$, extended by zero to $\mathbb R$, and define
\begin{equation}
    b(t) \coloneqq \frac{\eta(t)-\eta(1-t)}{\|\eta\|_\infty}.
\end{equation}
The two terms have disjoint supports. Hence $b\in C^\infty((0,1))$ is nonzero, $b(1-t)=-b(t)$, and $\|b\|_\infty\le 1$. In particular, $B_2>0$.

Choose a smooth cutoff $\chi\in C^\infty((1/8,5/8))$ satisfying $\chi=1$ on $[1/4,1/2]$, and a nonnegative function $\rho\in C^\infty((3/4,1))$ satisfying $\int_0^1\rho(x)\,dx=1$. Such functions exist by the standard smooth cutoff construction. Define on $[0,1]$
\begin{equation}
     w(x) = \chi(x) - \left(\int_0^1\chi(t)\,dt\right)\rho(x),
\end{equation}
and extend $w$ periodically. Since $w$ vanishes in neighborhoods of both endpoints, its periodic extension is smooth. Moreover, $\int_0^1w(x)\,dx=0$ and $w=1$ on $[1/4,1/2]$. For every integer $\ell\ge 1$, repeated integration by part in the trigonometric Fourier coefficients gives
\begin{equation}
    |\langle w,\phi_m\rangle| \le C_\ell m^{-\ell}, \qquad m\ge 1,
\end{equation}
where the boundary terms vanish by periodicity.
Choosing $\ell>\beta+1/2$ yields
\begin{equation}
     S_\beta \coloneqq \sum_{m\ge 1} m^{2\beta}|\langle w,\phi_m\rangle|^2 <\infty.
\end{equation}
Since $w$ is nonzero and has mean zero, $S_\beta>0$. Taking $\tau=R/\sqrt{S_\beta}$ proves the last assertion.
\end{proof}

From now, we fix a univariate centered function $g$, which belongs to the Sobolev ellipsoid of radius $R$ and smoothness $\beta$ and is constant to a real value $\tau$ on the compact $[ 1/4 , 1/2 ]$.

\begin{lemma}
\label{lem:e2e-holder}
Let $0 < h \le 1$, $ 0 < a \le 1$, and let define for all integer $k \geq 1$
\begin{equation}
    b_k(x) \coloneqq b\left( \frac{x - \nu_k}{h}\right),
\end{equation}
which have disjoint cells of length $h$, with $b$ from Lemma~\ref{lem:e2e-smooth-functions}, extended by zero. For arbitrary signs $\omega_k$, let define the function $u_{\bm\omega}$ as follows
\begin{equation}
    u_{\bm\omega} \coloneqq a h^\gamma\sum_k\omega_k b_k.
\end{equation}
For each integer $\ell\ge0$,
\begin{equation}
    \| (\sin u_{\bm\omega})^{(\ell)}\|_\infty
       \le C_\ell a h^{\gamma-\ell}.
\end{equation}
We also have
\begin{equation}
    \sup_{x\ne y} \frac{\left| (\sin u_{\bm\omega})^{(\lfloor \gamma \rfloor)}(x) - (\sin u_{\bm\omega})^{(\lfloor \gamma \rfloor)}(y) \right|}{|x-y|^{\gamma - \lfloor\gamma\rfloor}} \le C_\gamma a.
\end{equation}
These constants depend only on $b$, the indicated derivative order, and $\gamma$, and not on the number of cells or their signs.
\end{lemma}

\begin{proof}
On a single cell, write $z=a h^\gamma\omega_k$. For $|z|\le1$,
\begin{equation}
    \sin(zb(t))=\int_0^z b(t)\cos(vb(t))\,dv.
\end{equation}
Every derivative with respect to $t$ of the integrand is uniformly
bounded for $|v|\le1$, because $b$ is smooth with compact support.
Differentiation under this finite integral therefore bounds its
$\ell$th derivative by $C_\ell|z|$.
Rescaling $t=(x-\nu_k)/h$ multiplies this bound by $h^{-\ell}$.
All derivatives vanish at cell boundaries, and supports are
disjoint, so the same supremum bound holds globally.
For noninteger $\gamma$, put $r=\lfloor\gamma\rfloor$ and
$\delta=\gamma-r\in(0,1)$. The derivative of order $r$ has both supremum bound
$C a h^\delta$ and Lipschitz constant $C a h^{\delta-1}$.
Consequently,
\begin{equation}
    |(\sin u_{\bm\omega})^{(r)}(x)-(\sin u_{\bm\omega})^{(r)}(y)|
    \le C a\min\{h^\delta,h^{\delta-1}|x-y|\}
    \le C a|x-y|^\delta.
\end{equation}
The last inequality follows separately from $|x-y|\ge h$ and
$|x-y|\le h$. This proves the bound. At integer
$\gamma$, the already proved derivative estimate of order
$\gamma$ is $C a$, so its oscillation is at most $2Ca$.
\end{proof}

\subsection{Construction of alternatives}

In this subsection, we explicitly construct a finite subfamily $\{(p_{\boldsymbol\omega},F_{\boldsymbol\omega})\}_{\boldsymbol\omega} \subset\mathcal C_{\beta,\gamma}$ built on the previous subsection.

\begin{lemma}[Sine coupling and its exact marginals]
\label{lem:e2e-sine-marginals}
Let $u_1,\ldots,u_s$ be real functions on $[0,1]$ such that $\int\sin u_j=0$ and $\int\cos u_j=q$ for every $j$. For $s\le d$ and $0<\varepsilon<1$, set
\begin{equation}
    p(\mathbf x)=1+\varepsilon
        \sin\left(\sum_{j=1}^s u_j(x_j)\right).
\end{equation}
Then $p$ is a density with $1-\varepsilon\le p\le1+\varepsilon$,
and its marginals are
\begin{equation}
    \begin{cases}
        p_j = 1 + \varepsilon q^{s-1}\sin u_j, &\qquad (j\le s), \\
        p_j = 1, &\qquad (j>s).
    \end{cases}
\end{equation}
Furthermore, if we have $1\geq q \ge1-Aa^2h^{2\gamma}$, $Aa^2h^{2\gamma}\le1/2$, and $s h^{2\gamma}\le1$, then
\[
    \varepsilon e^{-2Aa^2}
    \le\varepsilon q^{s-1}\le\varepsilon.
\]
\end{lemma}

\begin{proof}
First, Fubini's theorem and the identity
$\exp(\mathrm{i}\sum_j u_j)=\prod_j\exp(\mathrm{i}u_j)$ give
\begin{equation}
    \int_{[0,1]^d}\exp\left(\mathrm{i}\sum_{j=1}^su_j(x_j)\right) \,d\mathbf x=q^s,
\end{equation}
which is real. Its imaginary part is zero, so $\int p=1$.

Second, the pointwise bounds follow from $|\sin|\le1$.

Third, integrating over all coordinates except an active coordinate $j$ gives imaginary part $q^{s-1}\sin u_j(x)$; integration over all active coordinates gives zero for an inactive marginal.

Finally, $\log(1-z)\ge-2z$ on $[0,1/2]$: its difference from
$-2z$ has derivative $(1-2z)/(1-z)\ge0$ and value zero at zero.
Apply this inequality to $z = 1-q$ to obtain $q^{s-1}\ge\exp(-2Aa^2(s-1)h^{2\gamma})\ge e^{-2Aa^2}$.
\end{proof}

In what follows, we fix $b,w,\tau,g=\tau w$ from Lemma~\ref{lem:e2e-smooth-functions}. In particular, $g=\tau$ on $[1/4,1/2]$, $\int g=0$, and $g$ is in the Sobolev ellipsoid of radius $R$. Choose once and for all
\begin{equation}
    0 < \varepsilon < \min\{1/2,1-\kappa_{\min},\kappa_{\max}-1\}.
\end{equation}
For integers $n,d\ge1$, set
\begin{equation}
    M = \lceil n^{1/(2\gamma+1)}\rceil,
    \qquad h=\frac1{4M},
    \qquad s=\min\{d,\lfloor h^{-2\gamma}\rfloor\}.
\end{equation}
Here $ s \ge 1 $, $Mh = 1/4$, and $ sh^{2\gamma} \le 1 $. 
For any integer $k \geq 1$, let define 
\begin{equation}
    \nu_k \coloneqq \frac{1}{4} + (k - 1)h,
\end{equation}
and set the function
\begin{equation}
    b_k( \cdot ) \coloneqq b \left( \frac{\cdot-\nu_k}h{}\right), \qquad 1\le k\le M,
\end{equation}
with zero extensions. All cells lie inside the plateau interval. Observe that if we define the bump supports $\mathcal B_k$ as follows
\begin{equation}
    \mathcal B_k \coloneqq \left[ \frac{1}{4} \left( \frac{k - 1}{M} + 1 \right) , \frac{1}{4} \left( \frac{k}{M} + 1 \right) \right[,
\end{equation}
we have $ \frac{x - \nu_k}{h} \in [0,1[ \iff x \in \mathcal B_k $ which implies that for any $k$, the functions $b_k$ reproduce the same function $b$ on disjoint intervals $\mathcal B_k$ and for any $ x \notin \mathcal B_k \implies b_k(x) = 0  $.

For a constant $a\in(0,1]$ to be selected below, and for a binary matrix $\bm\omega\in\{-1,1\}^{s\times M}$, define
\begin{equation}\label{eq:e2e-density-family}
    u_{\bm\omega,j}(x)=a h^\gamma\sum_{k=1}^M\omega_{jk}b_k(x),
    \qquad
    p_{\bm\omega}(\mathbf x) \coloneqq 1+\varepsilon\sin\left(\sum_{j=1}^s
                                      u_{{\bm\omega},j}(x_j)\right).
\end{equation}

\begin{lemma}
    We have the two following key results for all $j$:
\begin{eqnarray}
    \int_{[0,1]} \sin \left(u_{{\bm\omega},j}(x)\right) \, dx &=& 0, \\
    \int_0^1\cos \left(u_{{\bm\omega},j}(x)\right) \, dx  &=& \underbrace{1-\frac14\int_0^1[1-\cos(a h^\gamma b(t))]\,dt}_{\coloneqq q_h}.
\end{eqnarray}
\end{lemma}

\begin{proof}
    Let $\bm\omega$ be a vector of signs and define
    \begin{equation}
        u_{\bm\omega} \coloneqq a h^\gamma\sum_{k=1}^M\omega_{k}b_k. 
    \end{equation}

    Observe that $u_{\bm\omega}$ can be written as follows, using the bumps $\mathcal B_k$:
    \begin{equation}
        u_{\bm\omega}(x) = ah^\gamma \sum\limits_{k=1}^M \omega_k b_k(x) \mathbf{1}_{x \in \mathcal B_k}.
    \end{equation}

    Using that the $\mathcal B_k$ are disjoint and $\sin 0 = 0$, we have
    \begin{equation}
        \sin ( u_{\bm\omega}(x) ) = \sum\limits_{k=1}^M \omega_k \sin\left( a h^\gamma b_k(x) \right) \mathbf{1}_{x \in \mathcal B_k}.
    \end{equation}

    Observe that the function $b$ is defined on $[0,1]$, centered at $1/2$ and odd. Indeed, using $b(1 - t) = -b(t)$ we have $b(1/2) = 0$ and we also have $ b(x + 1/2) = b( 1 - (1/2 - x) ) = -b(1/2 - x) $. By translation, for each $k$, the function $b_k$ is odd on the bump $\mathcal B_k$ and centered on its middle, which gives the first equality.

    Moreover, we have
    \begin{equation}
        \cos ( u_{\bm\omega}(x) ) = \sum\limits_{k=1}^M \cos\left( a h^\gamma b_k(x) \right) \mathbf{1}_{x \in \mathcal B_k} + \prod\limits_{k=1}^M \mathbf{1}_{ x \notin \mathcal B_k }.
    \end{equation}
    In plain word, if $x \in \mathcal B_k$, the cosine equals $ \cos\left( a h^\gamma b_k(x) \right) $, else it equals $1$. Observe that 
    \begin{equation}
        \mathcal B_1 \cup \dots \cup \mathcal B_M = \left[ \frac{1}{4} , \frac{1}{2} \right[,
    \end{equation}
    so the volume of the set where the function equals $1$ is $3/4$, which yields
    \begin{equation}
        \int_{[0,1]} \cos ( u_\omega(x) ) \, dx = \frac{3}{4} + \sum\limits_{k=1}^M \int_{\mathcal B_k} \cos\left( a h^\gamma b_k(x) \right) \, dx.
    \end{equation}
    Computing the integral $ \int_{\mathcal B_k} \cos\left( a h^\gamma b_k(x) \right) \, dx $ is explicit by translating to $[0,1]$, which leads to the result.
    
\end{proof}

The basic inequality $1-\cos z\le z^2/2$ easily gives
\begin{equation}
    0\le1-q_h\le\frac{a^2B_2}{8}h^{2\gamma}\le\frac18.
\end{equation}
Applying Lemma~\ref{lem:e2e-sine-marginals}, we obtain valid
densities with the required strict joint bounds and exact marginals
\begin{equation}\label{eq:e2e-marginal-family}
    p_{\bm\omega,j}(x)=
    \begin{cases}
       1+t_h\sin u_{\bm\omega,j}(x),&j\le s,\\
       1,&j>s,
    \end{cases}
    \qquad
    t_h \coloneqq \varepsilon q_h^{s-1},
    \qquad
    \varepsilon e^{-a^2B_2/4}\le t_h\le\varepsilon.
\end{equation}
Lemma~\ref{lem:e2e-holder} gives the following uniform Hölder bound
\begin{equation}
    \forall x,y, \quad \left| p_{\bm\omega , j}^{(\lfloor \gamma \rfloor)}(x) - p_{\bm\omega , j}^{(\lfloor \gamma \rfloor)}(y) \right| \le C_\gamma \varepsilon a |x-y|^{\gamma - \lfloor\gamma\rfloor},
\end{equation}
where the constant $C_{\gamma}$ is fixed from now. Select $a>0$ small enough that this is at most $L$; the further restriction needed for testing will be imposed below. Thus every marginal belongs to $\mathcal H^\gamma(L)$.

Define the regression components by
\begin{equation}\label{eq:e2e-regression-family}
    f_{\bm\omega,j}=g/p_{\bm\omega,j}\quad(j\le s),
    \qquad f_{\bm\omega,j}=0\quad(j>s),
    \qquad f_{\bm \omega} = \sum_{j=1}^d f_{\bm\omega,j}.
\end{equation}
For each active coordinate $j$, we have $p_{\bm\omega,j}f_{\bm\omega,j} = g$ and
\begin{equation}
    \int_0^1 f_{\bm\omega,j}(x)p_{\bm\omega,j}(x)\,dx = \int_0^1g(x)\,dx=0.
\end{equation}
Inactive weighted components are zero. All weighted components therefore satisfy the Sobolev constraint, and every $(p_{\bm\omega},f_{\bm\omega})$ lies in $\mathcal C_{\beta,\gamma}$. The weighted signals remain fixed across the entire subfamily.

\subsection{Orthogonal separation of the regression functions}

Let define for all $x \in [0,1]$ and $k \geq 1$ the following functions
\begin{equation}
    z_k(x) \coloneqq \sin(a h^\gamma b_k(x)),\qquad
    H_k(x) \coloneqq \frac{\tau t_hz_k(x)}{1-t_h^2z_k(x)^2}.
\end{equation}
Only one bump can be nonzero at any $x$. On its support $g=\tau$, and identity $a^2 - b^2 = (a - b)(a + b)$ gives
\begin{equation}
    \frac1{1+t_h\omega_{jk}z_k} = \frac1{1-t_h^2z_k^2} -\frac{\omega_{jk}t_hz_k}{1-t_h^2z_k^2},
\end{equation}
which implies for all $\mathbf x \in [0,1]^d$
\begin{equation}\label{eq:e2e-affine-family}
    f_\omega(\mathbf x) = F_0(\mathbf x)-\sum_{j=1}^s\sum_{k=1}^M
         \omega_{jk}H_k(x_j),
\end{equation}
where the sign-independent function is explicitly
\begin{equation}
    F_0(\mathbf x) \coloneqq \sum_{j=1}^s\left[ g(x_j)+\tau\sum_{k=1}^M \frac{t_h^2z_k(x_j)^2}{1-t_h^2z_k(x_j)^2} \right].
\end{equation}
The map $z\mapsto \tau t_h z/(1-t_h^2z^2)$ is odd. Therefore, each $H_k$ is antisymmetric about the midpoint of its cell and satisfies
\begin{equation}
    \int_0^1 H_k(x)\,dx=0.
\end{equation}
The lifted functions $\mathbf x\mapsto H_k(x_j)$, indexed by $1\le j\le s$ and $1\le k\le M$, are pairwise orthogonal in $L^2(\lambda_d)$. Indeed, for a fixed coordinate $j$, distinct functions $H_k(x_j)$ have disjoint supports.
For distinct coordinates $j\ne j'$, Fubini's theorem gives
\begin{equation}
    \int_{[0,1]^d} H_k(x_j)H_{k'}(x_{j'})\,d\mathbf x = \left(\int_0^1 H_k(t)\,dt\right) \left(\int_0^1 H_{k'}(t)\,dt\right) = 0.
\end{equation}
Their squared norms are independent of $k$ and are equal to
\begin{equation}
    v_h^2 \coloneqq \tau^2t_h^2h\int_0^1
       \frac{\sin^2(a h^\gamma b(t))}
            {(1-t_h^2\sin^2(a h^\gamma b(t)))^2}\,dt.
\end{equation}
Since $|a h^\gamma b|\le1$, the inequalities $|\sin z|\ge |z|/2$ for $|z|\le1$ and $|\sin z|\le|z|$ give
\begin{equation}\label{eq:e2e-separation}
    \frac{\tau^2\varepsilon^2e^{-B_2/2}a^2B_2}{4}
           h^{2\gamma+1}
    \le v_h^2
    \le\frac{\tau^2\varepsilon^2a^2B_2}{(1-\varepsilon^2)^2}
           h^{2\gamma+1}.
\end{equation}
For clarity, the lower sine inequality follows by integrating
$\cos t\ge\cos1>1/2$ between $0$ and $|z|$; the upper one
follows by integrating $|\cos t|\le1$.
The lower exponential factor in \eqref{eq:e2e-separation} uses
$a\le1$ and the lower bound for $t_h$ in
\eqref{eq:e2e-marginal-family}.

\subsection{Information in the covariates and the responses}

Let $Q_{\bm\omega}$ denote the law of one observed pair under
$(p_{\bm\omega},f_{\bm\omega})$. Suppose $\bm\omega'$ differs from $\bm\omega$ only at $(j,k)$. The sine function is $1$-Lipschitz, and hence
\begin{align}
    \left|p_{\bm\omega}(\mathbf x)-p_{\bm\omega'}(\mathbf x)\right| &= \left| \varepsilon\sin\left(\sum_{j=1}^s u_{\bm\omega,j}(x_j)\right) - \varepsilon\sin\left(\sum_{j=1}^s u_{\bm\omega',j}(x_j)\right) \right|, \\
    &= \varepsilon\left| \sin\left(\sum_{j=1}^s u_{\bm\omega,j}(x_j)\right) - \sin\left(\sum_{j=1}^s u_{\bm\omega',j}(x_j)\right) \right|, \\
    &\leq \varepsilon \left| \sum_{j=1}^s \left( u_{\bm\omega,j}(x_j) - u_{\bm\omega',j}(x_j) \right) \right|, \\
    &\leq \varepsilon \left| \sum\limits_{j=1}^s \left( a h^\gamma \sum_{k=1}^M\omega_{jk}b_k(x_j) - ah^\gamma\sum_{k=1}^M\omega'_{jk}b_k(x_j) \right) \right|, \\
    &\leq \varepsilon a h^\gamma \left| \omega_{jk} - \omega'_{jk} \right| \left| b_k(x_j) \right|, \\
    &\leq 2 \varepsilon a h^\gamma | b_k(x_j) |.
\end{align}
Since $p_{\bm\omega'}\ge1-\varepsilon$, Lemma~\ref{lem:e2e-kl-bounds} yields
\begin{align}
    \operatorname{KL}(P_{\bm\omega} \| P_{\bm\omega'}) &\leq \int_{[0,1]^d} \frac{ \left( p_{\bm\omega}(\mathbf x) - p_{\bm\omega '}(\mathbf x) \right)^2 }{ p_{\bm \omega'}(\mathbf x) } \, d\lambda_d(\mathbf x), \\
    &\leq \int_{\mathcal B_k} \frac{ \left( 2 \varepsilon a h^\gamma b_k(x) \right)^2 }{ 1 - \varepsilon } \, dx, \\
    &\leq \frac{4\varepsilon^2a^2B_2}{1-\varepsilon} h^{2\gamma+1}.
\end{align}
Moreover, \eqref{eq:e2e-affine-family} gives exactly $f_{\bm\omega} -f_{\bm\omega'} = \pm2H_k(x_j)$. Thus
\begin{equation}
    \|f_{\bm\omega} - f_{\bm\omega'}\|_{L^2(P_{\bm\omega})}^2 \le 4\kappa_{\max}v_h^2.
\end{equation}
The Gaussian KL identity in Lemma~\ref{lem:e2e-gaussian-kl} and the previous inequalities give
\begin{equation}
    \operatorname{KL}(Q_{\bm \omega}^{\otimes n} \| Q_{\bm\omega'}^{\otimes n}) \leq \frac{4 n \varepsilon^2a^2B_2}{1-\varepsilon} h^{2\gamma+1} + \frac{2n\kappa_{\max}v_h^2}{\sigma^2}.
\end{equation}
Then, the inequality \eqref{eq:e2e-separation} gives
\begin{align}
    \operatorname{KL}(Q_{\bm \omega}^{\otimes n} \| Q_{\bm\omega'}^{\otimes n})
    &\le a^2 n h^{2\gamma+1} \underbrace{\left[
       \frac{4\varepsilon^2B_2}{1-\varepsilon}
       +\frac{2\kappa_{\max}\tau^2\varepsilon^2B_2}
              {\sigma^2(1-\varepsilon^2)^2} \right]}_{ C } \\
    &\le C n h^{2\gamma+1} a^2.
\end{align}
Moreover, by the choice of $M$ and the definition of $h$, we have
\begin{align*}
    n h^{2\gamma+1} &= n \left( \frac{1}{4M} \right)^{ 2\gamma + 1 } \\
    &\leq 4^{ -(2\gamma + 1) } \times n \times n^{ - \frac{2\gamma + 1}{2\gamma + 1} } \\
    &\leq 1
\end{align*}
All quantities inside the brackets are fixed constants. We finally select $a$ as follows
\begin{equation}
    0<a\le\min\left\{1,\frac{L}{C_\gamma\varepsilon}, \frac1{\sqrt{8C}}\right\},
\end{equation}
which ensures that
\begin{equation}
    \operatorname{KL}(Q_{\bm \omega}^{\otimes n} \| Q_{\bm\omega'}^{\otimes n}) \leq \frac{1}{8}.
\end{equation}

\subsection{Assouad's argument for the rough-density contribution}

For every estimator and vertex, the uniform bounds on $p$ give
\begin{equation}
    \|\widehat f - f_{\bm\omega}\|_{L^2(P_{\bm\omega})}^2 \ge\kappa_{\min} \|\widehat f - f_{\bm\omega}\|_{L^2(\lambda_d)}^2.
\end{equation}
If an estimator has infinite risk at a vertex, the desired lower bound already holds. Otherwise it can be regarded, up to null sets, as an $L^2(\lambda_d)$-valued estimator on this finite family. The observation laws here are mutually absolutely continuous, so the finitely many exceptional null sets can be
removed simultaneously.

Apply Lemma~\ref{lem:e2e-assouad} to
\eqref{eq:e2e-affine-family}, using directions $-H_k(x_j)$,
$N=sM$, and $\alpha=1/8$. Restricting the supremum to this
finite admissible family gives
\[
    \mathfrak M(n,d,\mathcal C_{\beta,\gamma})
    \ge\frac{3\kappa_{\min}}8 sM v_h^2
    \ge c\,s h^{2\gamma},
\]
where $Mh=1/4$ and \eqref{eq:e2e-separation} were used.
For $x\ge1$, $\lfloor x\rfloor\ge x/2$. Therefore we have
\begin{equation}
    s h^{2\gamma} \ge \frac12\min\{d h^{2\gamma},1\}.
\end{equation}
Finally, $n^{ \frac{1}{2\gamma + 1} }\le M \le2n^{\frac{1}{2\gamma + 1}}$ for every $n \ge 1$, so
\[
    8^{-2\gamma}n^{ - \frac{2\gamma}{2\gamma + 1} }
    \le h^{2\gamma}
    \le4^{-2\gamma}n^{- \frac{2\gamma}{2\gamma + 1}}.
\]
Combining the last three displays proves, for all integers $n,d\ge1$,
\begin{equation}\label{eq:e2e-rough-alone}
    \mathfrak M(n,d,\mathcal C_{\beta,\gamma})
    \gtrsim \min\{d \, n^{-\frac{2\gamma}{2\gamma + 1}},1\}.
\end{equation}

A sufficient condition to obtain the desired rate of $ \mathfrak M(n,d,\mathcal C_{\beta,\gamma}) \gtrsim d n^{-\frac{2\gamma}{2\gamma + 1}}$ is having $d\, n^{-\frac{2\gamma}{2\gamma + 1}} \leq 1 $ for $n$ large enough, so it suffices to assume
\begin{equation}
    d = o\left( n^{ \frac{2 \gamma}{ 2\gamma + 1 } } \right),
\end{equation}
which leads to the desired rate and tends to 0.

\end{document}